\documentclass[11pt]{article}

\RequirePackage[OT1]{fontenc}
\RequirePackage[numbers]{natbib}
\RequirePackage[colorlinks,citecolor=blue,urlcolor=blue]{hyperref}

\usepackage{amsthm}
\usepackage{amsmath}

\usepackage{amssymb}
\usepackage{mathtools}
\usepackage{bm}
\usepackage{verbatim}
\usepackage{algorithm}
\usepackage{algorithmic}
\usepackage{subcaption}
\usepackage{dsfont}

\theoremstyle{plain}
\newtheorem{theorem}{Theorem}[section]
\newtheorem{lemma}[theorem]{Lemma}
\newtheorem{proposition}[theorem]{Proposition}
\newtheorem{corollary}[theorem]{Corollary}
\theoremstyle{definition}

\newtheorem{remark}[theorem]{Remark}
\newcommand{\R}{\mathbb{R}}

\DeclareMathOperator*{\argmax}{arg\,max}

\newcommand{\eps}{\epsilon}
\newcommand{\sig}{\sigma}

\renewcommand{\P}{\mathsf{P}}
\newcommand{\E}{\mathsf{E}}

\newcommand{\inv}{^{-1}}
\newcommand{\lam}{\lambda}

\newcommand{\al}{\alpha}

\newcommand{\KL}{{D_\text{KL}}}
\newcommand{\dtv}{{D_\text{TV}}}

\newcommand{\ind}{\mathds{1}}

\newcommand{\ctv}{\nu}
\newcommand{\empctv}{{\widehat\nu}}
\newcommand{\Ph}{{\widehat \P}}

\renewcommand{\S}{\widetilde{\SS}}

\newcommand{\atv}{\mathcal{A}}

\newcommand{\VV}{\mathcal{V}}
\renewcommand{\SS}{\mathcal{S}}
\newcommand{\TT}{\mathcal{T}}
\newcommand{\WW}{\mathcal{W}}

\newcommand{\avinf}[3]{{\nu^{\texttt{avg}}_{#1|#2;#3}}}

\newcommand{\eavinf}[3]{{\widehat{\nu}^{\texttt{avg}}_{#1|#2;#3}}}

\newcommand{\elam}{\widehat{\lambda}}

\newcommand{\overbar}[1]{\mkern 1.5mu\overline{\mkern-1.5mu#1\mkern-1.5mu}\mkern 1.5mu}

\newcommand{\Ot}{\widetilde{\mathcal{O}}}

\newcommand{\OO}{\mathcal{O}}

\usepackage{changepage}

\begin{document}
\author{Guy Bresler \\ \\ Laboratory for Information and Decision Systems \\ 
Department of Electrical Engineering and Computer Science 
\\
Massachusetts Institute of Technology \\ \texttt{gbresler@mit.edu}}

\title{Efficiently learning Ising models on arbitrary graphs}

\date{}

\maketitle
\begin{abstract}
We consider the problem of reconstructing the graph underlying an Ising model from i.i.d. samples. Over the last fifteen years this problem has been of significant interest in the statistics, machine learning, and statistical physics communities, and much of the effort has been directed towards finding algorithms with low computational cost for various restricted classes of models. 
Nevertheless, for learning Ising models on general graphs with $p$ nodes of degree at most $d$, it is not known whether or not it is possible to improve upon the $p^{d}$ computation needed to exhaustively search over all possible neighborhoods for each node.

In this paper we show that a simple greedy procedure allows to learn the structure of an Ising model on an arbitrary bounded-degree graph in time on the order of $p^2$. We make no assumptions on the parameters except what is necessary for identifiability of the model, and in particular the results hold at low-temperatures as well as for highly non-uniform models. 
The proof rests on a new structural property of Ising models: we show that for any node there exists at least one neighbor with which it has a high mutual information. This structural property may be of independent interest. 
\end{abstract}

\thispagestyle{empty}

\clearpage
\setcounter{page}{1}

\section{Introduction}
Undirected graphical models, or Markov random fields, are a general and powerful framework for reasoning about high dimensional distributions and are at the core of modern statistical inference. The joint probability distribution specified by such a model factorizes according to an underlying graph, and the absence of edges encodes conditional independence \cite{lauritzen1996graphical}. The graph structure captures the computational aspect inherent in tasks of statistical inference including computing marginals, maximum \emph{a posteriori} assignments, the partition function, or sampling from the distribution. In addition to their statistical relevance, such computations on graphical models include as special cases many combinatorial optimization and counting problems. 

The Ising model is a Markov random field having binary variables and pairwise potential functions. The Ising model has a long and celebrated history starting with its introduction by statistical physicists as a model for spin systems in order to understand the phenomena of \emph{phase transition} \cite{ising1925beitrag,RevModPhys}. 
It has since been used across a wide spectrum of application domains including finance, social networks, computer vision, biology, and signal processing. The understanding of the computational tractability of inference (computing the partition function and sampling) has recently seen significant progress \cite{jerrum1993polynomial,Weitz,sinclair2014approximation,sly2010computational,sly2012computational, goldberg2003computational}.

The inverse problem of learning models from data is equally important. %, a subject with a large body of work in statistics and machine learning. 
Once the underlying graph is known it is relatively easy to estimate the parameters, hence the focus is largely on the task of \emph{structure learning}, i.e., estimating the graph.
Study of this problem was initiated by
Chow and Liu in their seminal 1968 paper~\cite{chow1968approximating}, which gave a greedy algorithm for learning \emph{tree-structured} Markov random fields with runtime on the order of $p^2$ for graphs on $p$ nodes. They showed that the maximum likelihood graph is a maximum-weight spanning tree, where each edge has weight equal to the mutual information of the variables at its endpoints. The maximum-likelihood tree can thus be found by a greedy algorithm, for example using Kruskal's or Prim's algorithms, and the running-time of the Chow-Liu algorithm is dominated by the time required to compute the mutual information between all pairs of nodes in the graph. 

For graphs with loops the problem is much more challenging for two reasons: a node and its neighbor can be \emph{marginally independent} due to indirect path effects, and moreover, this difficulty is compounded by presence of long-range correlations in the model, in which case distant nodes can be more correlated than nearby nodes. 
%In this paper our focus is on \emph{structure learning}, i.e. learning the graph underlying an Ising model; learning the parameters once the graph is known is relatively straightforward.
As discussed next in Subsection~\ref{ss:Background}, a basic first-order question has remained unanswered: it is not known if it is possible to learn the structure of Ising models on general graphs with $p$ nodes of degree at most $d$ in time less than $p^d$. This is roughly the time required to exhaustively search over all $p\choose d$ possible neighborhoods of a node and for each such candidate neighborhood test whether or not the implied conditional independence holds \cite{BMS08}.

In this paper we show that despite these challenges, a simple greedy algorithm allows to reconstruct arbitrary Ising models on $p$ nodes of maximum degree $d$ in time $\Ot(p^2)$, where the $\Ot(\cdot)$ notation hides a factor $\log p$ as well as a constant depending (doubly-exponentially) on $d$. Exponential dependence on $d$ is unavoidable as the number of samples must itself increase exponentially in $d$ \cite{santhanam2012information}, but doubly-exponential dependence on $d$ is probably suboptimal. 
The implication of our result is that learning Ising models on \emph{arbitrary graphs} of bounded degree has essentially the same computational complexity as learning such models on \emph{tree graphs}. 
The proof rests on a new basic property of Ising models: we show that for any node, there exists at least one neighbor with which it has high mutual information, even conditional on any subset of nodes. 
%This property is used in different ways both to show correctness of the algorithm and to bound its runtime.  

\subsection{Complexity of graphical model learning}\label{ss:Background}

A number of papers, including \cite{abbeel2006learning}, \cite{BMS08}, and~\cite{csiszar2006consistent}, 
have suggested to find each node's neighborhood by exhaustively searching over candidate neighborhoods and checking conditional independence.  
For graphical models on $p$ nodes of maximum degree $d$, such a search takes time (at least) on the order of $p^d$. As $d$ grows, the computational cost becomes prohibitive, and much work has focused on trying to find algorithms with lower complexity. 
Writing algorithm runtime in the form $f(d) p^{c(d)}$, for high-dimensional (large $p$) models the exponent $c(d)$ is of primary importance. We will think of efficient algorithms as having an exponent $c(d)$ that is bounded by a constant independent of $d$. 

Previous works proposing efficient algorithms either restrict the graph structure or the nature of the interactions between variables. Chow and Liu~\cite{chow1968approximating} made a model restriction of the first type, assuming that the graph is a tree; generalizations include to polytrees~\cite{dasgupta1999learning}, hypertrees~\cite{srebro2001maximum}, tree mixtures~\cite{anandkumar2012learningB},
and others. Among the many possible assumptions of the second type, the correlation decay property (CDP) is distinguished: 
until the recent paper~\cite{BGS14a}, all existing efficient algorithms required the CDP~\cite{bento2009graphical}. 
Informally, a graphical model is said to have the CDP if any two variables $\sig_i$ and $\sig_j$ are asymptotically independent as the graph distance between $i$ and $j$ increases. The CDP is known to hold for a number of pairwise graphical models in the so-called high-temperature regime, including Ising, hard-core lattice gas, Potts (multinomial), and others (see the survey article~\cite{gamarnikcorrelation} as well as, e.g., 
\cite{dobrushin1970prescribing,dobrushin1985constructive,salas1997absence,Weitz,gamarnik2007correlation,bandyopadhyay2008counting}).

It was first observed in \cite{BMS08} that it is possible to efficiently learn (in time $\widetilde{\OO}(p^2)$) models with (exponential) decay of correlations, under the additional assumption that neighboring variables have correlation bounded away from zero (as is true for the ferromagnetic Ising model in the high temperature regime). 
A variety of other papers including \cite{netrapalli2010greedy,raygreedy,anandkumar2012high} give alternative algorithms, but also require the CDP to guarantee efficiency. 
Structure learning algorithms that do not explicitly require the CDP are based on convex optimization, such as Ravikumar, Wainwright, and Lafferty's~\cite{ravikumar2010high} approach using $\ell_1$-regularized node-wise logistic regression. This algorithm has complexity $\OO(p^4)$; while it is shown to work under certain incoherence conditions that seem distinct from the CDP, Bento and Montanari \cite{bento2009graphical} established through a careful analysis that the algorithm fails for simple families of ferromagnetic Ising models without the CDP. Other convex optimization-based algorithms (e.g., \cite{lee2006efficient,jalali2011learning,jalali2011learning2})
assume similar incoherence conditions that are difficult to interpret in terms of model parameters, and likely also require the CDP. 

It is noteworthy that most computationally efficient \emph{sampling} algorithms (which happen to be based on MCMC) require a notion of temporal mixing, and this is closely related to spatial mixing or a version of the CDP (see, e.g., \cite{stroock1992logarithmic,dyer2004mixing,martinelli1994approach,Weitz}). The conclusion is that under a class of \emph{mixing conditions} (informally speaking), we can both generate samples efficiently as well as learn graphical models efficiently from i.i.d. samples. For antiferromagnetic Ising models on general bounded degree graphs, one has the striking converse statement that generating samples becomes intractable (NP-hard) precisely at the point where the CDP no longer holds \cite{sly2012computational}. 

Because all known efficient algorithms required the CDP, and because the Ising model exhibits dramatically different macroscopic behavior with versus without the CDP (and this determines computational tractability of sampling), Bento and Montanari~\cite{bento2009graphical} posed the question of whether or not the CDP is necessary for tractable structure learning. 
A partial answer was given in~\cite{BGS14a}, by demonstrating that a family of antiferromagnetic Ising models on general graphs can be learned efficiently despite strongly violating the CDP. Thus any relationship between the complexity of sampling (or computing the partition function) and the problem of \emph{structure learning} from i.i.d. samples seems tenuous, and this is corroborated by the results of this paper. In contrast, the recent papers \cite{BGS14b} and \cite{Montanari2014} demonstrate an algorithmic connection between \emph{parameter estimation} from \emph{sufficient statistics} for a model on a known graph and computation of the partition function.

\subsection{Results} 
We prove that the graph structure of an arbitrary Ising model on a bounded degree graph can be learned efficiently. 
Before discussing the algorithm, we first state the main conceptual contribution of the paper, namely identifying a new structural property of Ising models. Given the structural result the algorithm is almost obvious, and indeed it can be interpreted as a generalization of Chow and Liu's 1968 greedy algorithm for learning trees to models on arbitrary graphs.
%\GB{Discuss two conceptual contributions. Adding extra nodes, and not removing nodes. And identifying structural condition.}
\begin{proposition}[Structural property -- informal]
	Let $G$ be a graph on $p$ nodes of maximum degree $d$, and consider an Ising model on $G$. 
	Then for any node $u\in \VV$ there exists a neighbor $i$ such that the mutual information between $i$ and $u$ is at least some constant that is independent of $p$.
\end{proposition}

The property as stated is actually a consequence of Proposition~\ref{p:mutualInf}: we allow to condition on an arbitrary set of nodes, and instead of mutual information, we use a certain conditional influence measure.
As shown in Section~\ref{s:Correctness}, this influence provides a lower bound on the mutual information.

Our algorithmic result is stated in the following theorem.

%\GB{Put in dependence on $f(d)$ into $n$?}

\begin{theorem} [Algorithm performance -- informal]
Consider an Ising model on an arbitrary graph on $p$ nodes of maximum degree at most $d$. Given $n = f(d)\log p$ samples from the model, where the constant $f(d)$ is a function of the range of interaction strengths and $d$, it is possible to learn the underlying graph in time
$
f(d)p^2\log p\,.
$ 
\end{theorem}

We remark that the factor $f(d)$ in the sample complexity and runtime of our algorithm depends doubly-exponentially on $d$, in contrast to the necessary exponential dependence discussed in Subsection~\ref{ss:GMlearning}. 
While the focus of this paper is not on sample complexity, our algorithm does obtain (with a suboptimal constant) the optimal logarithmic dependence on $p$. 

Our algorithm is described in Section~\ref{s:Alg} and a
more detailed statement of the theorem (with full dependence on constants) is given as Theorem~\ref{t:mainFull}. The algorithm is extremely simple and quite natural: in order to find the neighborhood of a node $u$, it greedily (according to a measure of conditional influence) adds nodes one-by-one to form a constant-size superset of the neighborhood
(pseudo-neighborhood). The idea of adding spurious nodes to form a superset is not new, but it is worth emphasizing the difference relative to algorithms that attempt to add only correct nodes.

It is also crucial that the algorithm only adds---and does not remove---nodes in creating the pseudo-neighborhood. The reason is that in models with long-range correlations, a non-neighbor $i$ with high influence on $u$ (or high mutual information) contains a lot of information about many \emph{other non-neighbors}, so conditioning on $i$ effectively eliminates many non-neighbors from consideration. This allows us to use a potential argument, whereby each added node reduces the conditional entropy of the node $u$ by some constant, and this bounds the size of the pseudo-neighborhood. 
The pseudo-neighborhood can then be easily pruned to remove non-neighbors.

%However, in order to obtain good \emph{statistical} efficiency, it is important that our measure of influence is \emph{signed}, so that fluctuations in the empirical distribution are averaged out. This is discussed in more detail in Subsection~\ref{ss:influence}.

We next mention a few connections to other work and then in Section~\ref{s:Preliminaries} define the Ising model and structure learning problem. Section~\ref{ss:influence} introduces the notion of influence we use and states a lemma showing that empirical estimates are close to the population averages. Section~\ref{s:Alg} presents our algorithm with performance guarantee stated in Theorem~\ref{t:mainFull}. Section~\ref{s:Correctness} contains proofs of correctness and runtime, as well as a statement of our structural result, Proposition~\ref{p:mutualInf}. The proposition is proved in Sections~\ref{s:PropProof} and~\ref{s:techLemma}. Finally, Section~\ref{s:EmpiricalLemma} contains the proof of the lemma from Section~\ref{ss:influence} and Section~\ref{s:discussion} discusses possible extensions.

\subsection{Other related work}
%Estimation of graphical models 
%has been studied by a number of communities. 
Since the 1980's, Hinton and others have studied the problem of learning Ising models under the name of learning Boltzmann machines \cite{ackley1985learning,hinton1986learning,tanaka1998mean}. Most approaches for learning Boltzmann machines do not assume a sparse underlying graph and attempt to find parameters directly, using gradient optimization methods.
Ising models are used to model neuronal networks and protein interaction networks, and the problem of learning Ising models has been studied in this context
\cite{schneidman2006weak,cocco2009neuronal,mora2010maximum,weigt2009identification}. 
The problem of learning Ising models has also been of interest in the statistical physics community, where it is known as the inverse Ising problem. A variety of non-rigorous methods have been proposed, including based on truncation of expansions relating couplings to mean parameters or entropies \cite{sessak2009small,cocco2012adaptive,roudi2009statistical,lezon2006using}, message-passing \cite{mezard2009constraint,ricci2012bethe,aurell2010dynamics}, and others~\cite{decelle2014pseudolikelihood}.

%Protein interaction networks \cite{weigt2009identification}

%The general form of the Ising model we consider, with arbitrary edge couplings, includes as a special case various models of spin glasses where the edge couplings are sampled independently from some distribution. So-called \emph{diluted} spin glasses are models on random $d$-regular graphs, and thus our results imply that such models can be learned efficiently from i.i.d. samples. These models exhibit subtle and complicated behavior, and recent years have seen significant progress in their understanding in a series of deep results in probability and statistical physics \cite{talagrand2011mean}.

Structure learning of graphical models has been studied in the statistics and machine learning communities as a problem in high-dimensional statistics. Broadly, in high-dimensional statistics one wishes to estimate a high-dimensional object from samples, where the number of samples is far less than the dimensionality of the parameter space. To solve this ill-determined problem requires that there be some sparse underlying combinatorial structure. In our case the graph underlying the Markov random field is a sparse graph. As discussed in \cite{negahban2012unified}, optimization of regularized objective functions has been a popular approach to many problems in high-dimensional statistics including sparse linear regression, low-rank matrix completion, inferring rankings, as well as the problem of learning graphical models (both binary and Gaussian) \cite{ravikumar2010high,meinshausen2006high,friedman2008sparse,ravikumar2011high}. Such general methodology based on optimizing likelihood or (pseudo-likelihood) has failed thus-far to learn Ising models on general graphs, and one of the messages of our work is that a tailored approach is often necessary.

The theoretical computer science community has made progress on learning a variety of high-dimensional probability distributions from samples, including mixtures of Gaussians \cite{moitra2010settling,belkin2010polynomial}. 
But there is a more intriguing connection to work on learning function classes. In an Ising model, the conditional distribution of a node given its neighbors is specified by a logistic function, which is a soft version of a linear threshold function (LTF). Thus our algorithm effectively learns soft LTFs over a complicated joint distribution. 
Arguments based on boolean Fourier analysis have played a major role in learning boolean functions over uniformly random examples and also over product distributions \cite{kalai2009learning}, but
due to the complicated joint dependencies in an Ising model, it is not obvious how to apply Fourier analysis in our setting. Our structural result, nevertheless, is at a high level analogous to the statement that LTFs have non-trivial total degree-one Fourier mass (see, e.g., \cite{o2014analysis}). 
Other recent works learn LTFs over non-product distributions, including log-concave \cite{kane2013learning} and sub-Gaussian or sub-exponential \cite{klivans2013moment}. These assumptions are badly violated by Ising models in the low-temperature regime (with long-range correlations), but the bounded graph degree assumption makes our soft LTFs depend on few variables, and this makes learning tractable.

\section{Preliminaries}\label{s:Preliminaries}
\newcommand{\EE}{\mathcal{E}}
\subsection{Ising model}
%For $\sigma\in \{-1,+1\}^p$,
We consider the Ising model on a graph $G=(\VV,\EE)$ with $|\VV|=p$. The notation $\partial i$ is used to denote the set of neighbors of node $i$, and the degree $|\partial i|$ of each node $i$ is assumed to be bounded by $d$. 
To each node $i\in \VV$ is associated a binary random variable (spin) $X_i$. Each configuration of spins $x\in \{-,+\}^\VV$ (`$-$' and `$+$' are used as shorthand for $-1,+1$) is assigned probability according to the probability mass function 
\begin{equation}\label{e:Ising}
P(x) =  \exp\bigg(\sum_{\{i,j\}\in \EE}\theta_{ij} x_ix_j + \sum_{i\in \VV}\theta_i x_i - \Phi(\theta)\bigg)\,.
\end{equation}
Here $\Phi(\theta)$ is the log-partition function or normalizing constant. The distribution is parameterized by the vector $\theta = \{\theta_{ij}\}_{\{i,j\}\in \EE} \cup \{\theta_i\}_{i\in \VV}\in \R^{\EE\cup \VV}$, consisting of edge couplings and node-wise external fields. The edge couplings are assumed to satisfy 
$
\alpha \leq |\theta_{ij} | \leq \beta$ for all $\{i,j\}\in \EE
$
for some constants $0<\al\leq \beta$ and the external fields are assumed to satisfy 
$
|\theta_i|\leq h$ for all $i\in \VV\,.
$ The bounds $\alpha,\beta,h$ are necessary for model identifiability, and as shown in~\cite{santhanam2012information} and discussed briefly in the next subsection, must appear in the sample complexity.

We can alternatively think of $\theta\in \R^{{p\choose 2}+p}$, with $\theta_{ij}=0$ if $\{i,j\}\notin \EE$. For a graph $G$, let  
\begin{align*}
\Omega_{\alpha,\beta,h}(G) = \{\theta\in \R ^{{p\choose 2}+p} : |\theta_i|\leq h \text{ for } i\in \VV, \alpha\leq |\theta_{ij}|\leq \beta\text { if } \{i,j\}\in \EE,  \text { and }\theta_{ij} = 0  \text{ otherwise}\}
\end{align*}
be the set of valid parameter vectors corresponding to~$G$.

The distribution specified in $\eqref{e:Ising}$ is a \emph{Markov random field}, and an implication is that each node is conditionally independent of all other nodes given the values of its neighbors.
The conditional probability of $X_u=+$ given the states of all the other nodes $\VV\setminus \{u\}$ can thus be written as:
\begin{align}
\P(X_u=+\,|X_{\VV\setminus \{u\}}=x_{\VV\setminus \{u\}}) = \P(X_u=+\,|X_{\partial u}=x_{\partial u})=
 \frac{\exp\big(2 \sum_{i\in \partial u} \theta_{ui} x_i +\theta_u\big)}{1+\exp\big(2 \sum_{i\in \partial u} \theta_{ui} x_i+ \theta_u\big)}
\,.\label{e:conditional}
\end{align}
%In writing the conditional probability, conditioning on $x_{\VV\setminus \{u\}}$ is shorthand for conditioning on the event $\{X_{\VV\setminus \{u\}}=x_{\VV\setminus \{u\}}\}$.

A useful property of bounded degree models is that the conditional probability of a spin is always bounded away from $0$ and $1$.
The proof of this statement is immediate from~\eqref{e:conditional} by conditioning on the neighbors of $u$ and using the tower property of conditional expectation.
\begin{lemma}[Conditional randomness]\label{l:condRandomness}
For any node $u\in \VV$, subset $\SS\subseteq \VV\setminus \{u\}$, and any configuration $x_\SS\in \{-,+\}^{|\SS|}$,
$$\min \{\P(X_u = +|X_\SS = x_\SS), \P(X_u = -|X_\SS = x_\SS)\}\geq \frac12 e^{-2(\beta d+h)}:=\delta\,.$$
This quantity is denoted by $\delta$ and appears throughout the paper.
\end{lemma}

\newcommand{\GG}{\mathcal{G}}
\newcommand{\Gpd}{{\GG_{p,d}}}

\subsection{Graphical model learning}\label{ss:GMlearning}

%As mentioned before, the underlying graph $G$ is assumed to have maximum node degree bounded by $d$, and w

Denote the set of all graphs on $p$ nodes of degree at most $d$ by $\Gpd$.
For some graph $G\in \Gpd$ and parameters $\theta\in \Omega_{\al,\beta,h}(G)$, one observes configurations $X^{(1)},\dots, X^{(n)}\in \{-,+\}^p$ sampled independently from the Ising model~\eqref{e:Ising}. 
 A \emph{structure learning algorithm} is a (possibly randomized) map
$$
\phi:(\{-1,+1\}^p)^n\to \Gpd
$$
taking $n$ samples $X^{1:n}=X^{(1)},\dots, X^{(n)}$ to a graph $\phi(X^{1:n})$. 
The statistical performance of a structure learning algorithm will be measured using the zero-one loss, meaning that the exact underlying graph must be learned.
The risk, or expected loss, under some vector $\theta\in \Omega_{\al,\beta,h}(G)$ of parameters corresponding to a graph $G\in \Gpd$ is given by the probability of reconstruction error
$$
\P_{\theta} (\phi(X^{1:n})\neq G)\,,
$$
and for given $\al,\beta, h,p,d$, the maximum risk is
$$
\sup_{G\in \Gpd\atop \theta\in\Omega_{\al,\beta,h}(G)}\P_{\theta} (\phi(X^{1:n})\neq G)\,.
$$
Our goal is to find an algorithm with maximum risk (probability of error) tending to zero as $p\to \infty$, using the fewest possible number of samples $n$. This notion of performance is rather stringent, but also robust, being worst-case over the entire class of graphs $\Gpd$ and parameters $\theta\in \Omega_{\al,\beta,h}(G)$.
A lower bound on the number of samples necessary was obtained by Santhanam and Wainwright in Theorem~1 of~\cite{santhanam2012information}:
%$$
%n\geq \max \Big\{
%\frac{\log p }{2\al \tanh \al}, 
%\frac{e^{\beta d}\log \big(\frac{pd}4-1\big)}{4\al d e^\al  }, \frac{d}8\log \frac{p}{8d}
%\Big\}\,.
%$$
$$
n\geq 
\frac{e^{\beta d}\log \big(\frac{pd}4-1\big)}{4\al d e^\al  }\,.
$$
This means in particular that exponential dependence of the sample complexity (and hence runtime) on the quantity $\beta d$ is unavoidable. 
%The minimum number of samples such that an algorithm has risk tending to zero is the algorithm's sample complexity.

%The minimax risk is the best algorithm's worst-case risk (probability of error) over graphs and corresponding parameter vectors, namely
%$$
%R_{p,d,n}\triangleq\min_\phi\max_{G\in \Gpd\atop \theta\in \Omega(G)} \P_{\theta} (\phi(X)\neq G)\,.
%$$ 
%The basic question we seek to address is what triples $n,p,d$ result in the minimax risk $R_{p,d,n}$ tending to zero as these parameters tend to infinity.

\section{Measuring the influence of a variable}\label{ss:influence}

%\begin{equation}
%H(\sig_A) = -\sum_{x_A\in \{-1,+1\}^p} \P(\sig_A=x_A) \log\P(\sig_A=x_A)
%\end{equation}
%
%
%\begin{equation}
%H(\sig_A|\sig_B) = -\E_{\sig_B}\sum_{x_A\in \{-1,+1\}^p} \P(\sig_A=x_A|\sig_B) \log\P(\sig_A=x_A|\sig_B)
%\end{equation}
%
%\begin{equation}\label{e:mutInf}
%I(\sig_A;\sig_B|\sig_C)= H(\sig_A|\sig_C) - H(\sig_A|\sig_B,\sig_C)
%\end{equation}
%
%The Kullback-Leibler divergence is given by
%$
%\KL(P\|Q) = \E_{x\sim P} \log \frac {P(x)}{Q(x)}\,.
%$
%
%\begin{equation}
%I(\sig_A;\sig_B|\sig_C) = \KL(\P(\sig_A  \sig_B\,| \sig_C)\| \P(\sig_A | \sig_C) \P( \sig_B| \sig_C))
%\end{equation}
%
%%\begin{equation}
%%I(\sig_A;\sig_B|\sig_C) = \KL(\P(\sig_A\in \cdot \,, \sig_B\in \cdot\,| \sig_C)\| \P(\sig_A\in \cdot\, | \sig_C) \P( \sig_B\in \cdot\,| \sig_C))
%%\end{equation}
%
%Mutual information is non-negative and is bounded by entropy:
%$
%0\leq I(\sig_A;\sig_B|\sig_C)\leq H(\sig_A|\sig_C)
%$.
%Chain rule for mutual information:
%$
%I(\sigma_u;\sigma_{1},\dots,\sigma_r) = \sum_{\ell = 1}^{r} I(\sigma_u;\sigma_\ell|\sigma_{1},\dots,\sigma_{\ell-1})
%$
%Total variation distance
%$
%\dtv(P,Q) = \sup_{A} | P(A) - Q(A)| = \frac12 \sum_{x\in \XX} |P(x) - Q(x)|\,.
%$
%Pinsker's inequality:
%$
%\dtv(P,Q) \leq \sqrt{\frac12 \KL(P\|Q)}\,.
%$

Our algorithm uses a certain \emph{conditional influence} of a variable on another variable. 
For nodes $u,i\in \VV$, subset of nodes $\SS\in \VV\setminus \{u,i\}$ and configuration $x_\SS\in \{-,+\}^\SS$, define  
\begin{align*}
\ctv_{u|i;x_\SS}&:=  \P(X_u=+|X_i=+,X_\SS = x_\SS)-\P(X_u=+|X_i=-,X_\SS = x_\SS)\,.
\end{align*}
%and
%\begin{align*}
%\ctv_{u|i;\SS}&:=  \min_{x_\SS}|\ctv_{u|i;x_\SS}|\,.
%\end{align*}
We also use a quantity we call the \emph{average conditional influence},
which is obtained by performing a weighted average of $|\ctv_{u|i;X_\SS}|$ over random configurations $X_\SS$:
$$
\avinf ui\SS:= \E \big(\lam_i(X_\SS)\cdot |\ctv_{u|i;X_\SS}| \big)\,.
$$
The weights are given by the function
$$
\lam_i(x_\SS) = 2\cdot \P(X_i=+|X_\SS = x_\SS) \P(X_i=-|X_\SS = x_\SS)\,.
$$

By the Markov property \eqref{e:conditional}, the influence is zero for non-neighbors conditional on neighbors, that is, for any $x_\SS$,
\begin{equation}
\label{e:NonNbr}
\ctv_{u|i;x_\SS}=0
 \qquad \text{for all } i\in \VV\setminus\{u,\SS\} \text{ if } \partial u \subseteq \SS\,.
\end{equation}
This implies the same statement for $\avinf u i\SS$.
Our structural result, Proposition~\ref{p:mutualInf}, shows that $\avinf u i\SS$ is bounded below for at least \emph{one neighbor} $i\in \partial u\setminus \SS$ if $\SS$ does not already contain the neighborhood $\partial u$. Thus computing $\avinf u i\SS$ allows to determine if $\partial u\subset \SS$ or not, and our algorithm given in Section~\ref{s:Alg} is based on this idea.   

%\begin{remark} 
%Because we require conditioning on sets $\SS$ that are relatively large (exponential in $\beta d$), obtaining sufficiently accurate estimates of   $|\ctv_{u|i;x_\SS}|$ 
%for our algorithm to succeed requires number of samples to grow \emph{doubly} exponential in~$d\beta$. 
%\end{remark}

\begin{remark}
Other works using a ``conditional independence test", for example~\cite{anandkumar2012high,wu2013learning}, use a similar measure of influence amounting to $\min_{x_\SS}|\ctv_{u|i;x_\SS}|$. We do not take the minimum over configurations $x_\SS$, as there is no guarantee that $\min_{x_\SS}|\ctv_{u|i;x_\SS}|$ is nonzero for any neighbor $i\in \partial u$: each $X_i$ can be marginally independent of $X_u$ conditional on \emph{some} configuration $x_\SS$. 
\end{remark}

The \emph{empirical conditional influence} $\empctv_{u|i;x_\SS}$ replaces the probability measure by the empirical measure:
\begin{equation*}
\empctv_{u|i;x_\SS}:=\Ph(X_u=+|X_i=+,X_\SS = x_\SS)-\Ph(X_u=+|X_i=-,X_\SS = x_\SS)\,,
\end{equation*}
where for $\SS,\TT\subseteq \VV$, 
$$
\Ph (X_\TT=x_\TT|X_\SS = x_\SS) =\frac{\Ph (X_\TT=x_\TT,X_\SS = x_\SS)}{\Ph(X_\SS = x_\SS)}\qquad\text{and}\qquad
\Ph(X_\SS=x_\SS)=
 \frac1n\sum_{t=1}^n \ind_{\{X_\SS^{(t)}=x_\SS\}}\,.
$$
%As before we let $\empctv_{u|i;\SS} = \min_{x_\SS}|\empctv_{u|i;x_\SS}|$.
Like before we define an averaged version (with average taken according to the empirical measure):
$$\eavinf u i\SS:=\E_{X_\SS\sim \Ph}
\Big(\elam_i(X_\SS)\cdot \empctv_{u|i;X_\SS}\Big)\,,$$
where
$$
\elam_i(x_\SS) = 2\cdot \Ph(X_i=+|X_\SS = x_\SS) \Ph(X_i=-|X_\SS = x_\SS)\,.
$$

It will be necessary that these empirical influences are sufficiently accurate. %for sets $\SS$ up to some constant size $\ell$. 
Let $\atv(\ell,\eps)$ denote the event that empirical influences with conditioning set up to size $\ell$ 
are accurate to within an additive $\eps$:
$$
\atv(\ell,\eps) = \{|\avinf ui\SS- \eavinf ui\SS|\leq \eps \text{ for all } \{u,i\}\subset \VV, \SS\subset \VV\setminus\{u,i\}, |\SS|\leq \ell\}\,.
$$
%We remark that the algorithm does not need the conditional influences to be accurate for all $x_\SS$, and relaxing this requirement would allow to show correctness with somewhat fewer samples than stated in the following lemma. Since the improvement is not drastic, and
%our focus is on computational complexity, we keep the arguments as simple as possible at the cost of worse sample complexity. 

\begin{lemma}\label{l:closeTV} Recall the notation $\delta:=\frac12 e^{-2(\beta d+h)}$
 and suppose $\ell \leq p/4-2$. If the number of samples $n$ is at least $\frac{144(\ell+3)}{\eps^2\delta^{2\ell}}\log\frac p \zeta$, then $\P(\atv(\ell,\eps))\geq 1-\zeta$. 
\end{lemma}

The proof the lemma follows from Azuma's inequality and is similar to Theorem~2 of~\cite{BMS08}. We give the proof in Section~\ref{s:EmpiricalLemma}.

\section{Algorithm}\label{s:Alg}
We now describe the structure learning algorithm, which learns the neighborhood of an arbitrary individual node~$u\in \VV$.
Algorithm \textsc{LearnNbhd} takes as input the node~$u$ whose neighborhood~$\partial u$ we wish to learn as well as the data $X^{(1)},\dots, X^{(n)}$ and a threshold parameter~$\tau$. The first step is to construct a \emph{superset} $\SS$ (which we call a pseudo-neighborhood) of the neighborhood $\partial u$. This is accomplished by greedily adding to $\SS$ the node $i$ with highest conditional influence $\eavinf ui\SS$, until there are no nodes $i$ with $\eavinf ui\SS\geq \tau$. 
(To simplify the description we set $\eavinf ui\SS=0$ if $i\in \SS\cup\{u\}$.)
At each step the conditional influences are computed with respect to the current set $\SS$. 

As will become apparent in Subsection~\ref{e:entropyIncrement}, inclusion of non-neighbors is important, as it allows to 
%This is because in models with long-range correlations, a non-neighbor $i$ with high influence on $u$ (or high mutual information) contains a lot of information about many \emph{other non-neighbors}, so conditioning on $i$ effectively eliminates many non-neighbors from consideration. 
use a potential argument to show that the constructed pseudo-neighborhood is not too large. Concretely, by definition of the algorithm and a simple lemma relating influence to mutual information, adding a node to the pseudo-neighborhood $\SS$ always reduces by at least $\tau$ the conditional entropy $H(X_u|X_\SS)$ of the variable $X_u$ whose neighborhood we are trying to find. The entropy is non-negative and was initially at most one (since $X_u$ is binary), so this bounds the size of $\SS$. 

The correctness of the algorithm relies on Proposition~\ref{p:mutualInf} of Subsection~\ref{ss:KeyStruct}, which shows that there is always at least one neighbor with influence above a constant, and we set $\tau$ equal to this constant. This implies that the algorithm does not terminate before all the neighbors are added. 
Finally, after construction of the pseudo-neighborhood, the algorithm removes those nodes with low average influence. Proposition~\ref{p:mutualInf} is again used to show that no neighbors are removed.

\begin{center}
    \begin{minipage}{.55\textwidth}
        \begin{algorithm}[H]
        \caption{\textsc{LearnNbhd}($X^{(1)},\dots, X^{(n)},\tau,u$)}
\texttt{\emph{Pseudo-neighborhood:}
        \begin{adjustwidth}{.3cm}{}
         \texttt{1. Let}  $\SS=\varnothing$ \\
        \texttt{2. Let} $(i^*,\eta^*) = (\argmax_i \eavinf ui\SS, \max_i \eavinf ui\SS)$\\
              \texttt{3. If} $\eta^*\geq \tau$, then add $i^*$ to $\SS$
              \\ \texttt{4.}\qquad  Else goto Step~6 \\
              \texttt{5.} Repeat Steps 2 to 4
 \end{adjustwidth}  
\emph{Pruning:}
        \begin{adjustwidth}{.3cm}{}
         \texttt{6. For each} $i\in \SS$ \texttt{if} $\eavinf ui{\SS\setminus\{i\}}<\tau$ remove $i$ 
          \\
          7. \texttt{Output} $\SS$
         \end{adjustwidth} }
        \end{algorithm}
        \vspace{.2cm}
    \end{minipage}
\end{center}

\begin{theorem}\label{t:mainFull}
Let $G\in \Gpd$ and $\theta\in \Omega_{\al,\beta,h}(G)$.
Let $\delta = \frac12 e^{-2(\beta d+h)}$ and define $$\tau^*=\frac{\al^2 \delta^{4d+1} }{16d\beta}\,,\qquad \eps^* = \frac{\tau^*}2\,,\qquad \ell^*=\frac2{(\tau^*-\eps^*)^2}=\frac{8}{(\tau^*)^2}\,.$$
Suppose we observe $n$ samples $X^{(1)},\dots,X^{(n)}$, for $$n\geq \frac{144(\ell^*+3)}{ (\eps^*)^2 \delta^{2\ell^*}}\log \frac{p}{\zeta}
= \exp\{c\alpha^{-c'}e^{c''d(\beta d+h)}\} \cdot \log \frac{p}{\zeta}
\,,$$
where $c,c',c''$ denote numerical constants. 
%Let $\tau = \tau^*/2$.
Then with probability at least $1-\zeta$, the structure learning algorithm \textsc{LearnNbhd}$(X^{1:n},\tau^*,u)$  returns the correct neighborhood $\partial u$ for all $u\in \VV$ and the runtime for each of the $p$ nodes is (for a numerical constant C)
$$C\ell^* p n = \OO(p\log p)\,.$$
\end{theorem}
 
\begin{remark}
As stated, the algorithm has probability $1-\zeta$ of both returning the correct neighborhoods for all nodes and having the claimed runtime. Obviously, the algorithm can be terminated if the runtime exceeds the stated value, giving a deterministic guarantee on runtime.  
\end{remark}

\section{Algorithm correctness}\label{s:Correctness}

In this section we prove Theorem~\ref{t:mainFull}, first giving a bound on the run-time in Subsection~\ref{e:entropyIncrement} and then showing correctness in Subsection~\ref{ss:KeyStruct}.

\subsection{Entropy increment argument and run-time bound}
\label{e:entropyIncrement}

In this subsection we bound the size of the pseudo-neighborhood constructed, but make no guarantee that it actually contains the true neighborhood. We use several standard information-theoretic quantities including entropy, Kullback-Leibler divergence, and mutual information. The relevant definitions can be found in any information theory textbook (such as \cite{cover2012elements}). The following lemma gives a lower bound on the conditional mutual information of each added node. 

\begin{lemma}
\label{l:mutInfAdded}
Assume that event $\atv(\ell,\eps)$ holds and suppose that \textsc{LearnNbhd} added node ${j_{\ell+1}}$ to the pseudo-neighborhood of $u$ after having added $j_1,\dots,j_{\ell}$.
Then the conditional mutual information $I(X_u;X_{j_{\ell+1}}|X_{j_1},\dots,X_{j_{\ell}})\geq \frac12(\tau - \eps)^2$.
\end{lemma}
We can now argue that the number of nodes added in the pseudo-neighborhood step is not too large by using a potential argument. The bound is stated in terms of the quantity
$$
\ell(\tau,\eps) = \frac{2}{(\tau - \eps)^{2}}\,.
$$

\begin{lemma}\label{l:PseudoSize}
If event $\atv(\ell(\tau,\eps),\eps)$ holds, then
at the end of the pseudo-neighborhood construction step the set $\SS$ has cardinality at most $\ell(\tau,\eps)=2(\tau - \eps)^{-2}$.
\end{lemma}
Before proving the lemma, let us quickly see how it justifies the runtime claimed in Theorem~\ref{t:mainFull}. Each maximization step in line~2 of \textsc{LearnNbhd} takes time $\OO(pn)$, so we get a cost of $\OO(|\SS| pn)=\OO(\ell^* pn)$ for the pseudo-neighborhood step, and this dominates the runtime. 

\begin{proof}[Proof of Lemma~\ref{l:PseudoSize}]
%Adding a node to the pseudo-neighborhood reduces the remaining entropy of the spin $X_u$ by the conditional mutual information, which is at least $ \frac12 (\tau - \eps)^2$, and since the entropy is non-negative this implies termination of the algorithm. In more detail,
Consider the sequence of nodes added, $j_1,j_2,\dots,j_r$, with $\SS=\{j_1,j_2,\dots,j_r\}$. Then
\begin{align*}
1\geq H(X_u) \stackrel{(a)}{\geq} I(X_u|X_{\SS})   \stackrel{(b)}{=}
\sum_{k=1}^{r} I(X_u;X_{j_k}|X_{j_1},\dots,X_{j_{k-1}})
\stackrel{(c)}{\geq }
\min\{\ell+1,r\}\cdot \tfrac12 (\tau - \eps)^2 \,.
\end{align*}
(a) is by non-negativity of conditional entropy and the definition of mutual information, $I(X_u|X_\SS) = H(X_u) - H(X_u|X_\SS)$, 
(b) follows by the chain rule for mutual information, and (c) is by Lemma~\ref{l:mutInfAdded}. Since $\ell+1$ is strictly larger than $ 2(\tau - \eps)^{-2}$, $r=|\SS|$ must satisfy the bound stated in the lemma.
\end{proof}

\begin{proof} [Proof of Lemma~\ref{l:mutInfAdded}]
Let ${\SS_\ell}= ({j_1},\dots,{j_{\ell}})$ consist of the first $\ell$ nodes already added to the pseudo-neighborhood. Our goal is to show that the next node $j_{\ell+1}$ has mutual information $I(X_u;X_{j_{\ell+1}}|X_{\SS_\ell})\geq  \frac12 (\tau - \eps)^2$. 
%We will show the stronger statement that for each configuration $x_S$, 

We use a shorthand for the (random) conditional measure: $Q(u+) = \P( X_u = + | X_{\SS_\ell})$ and similarly $Q(u+|i-) = \P( X_u = + | X_i = -,X_{\SS_\ell} )$, with similar definitions for any combination of `$+$' and `$-$'. 
%For a set of nodes $\WW$ and configuration $x_\WW$ we will also write $q(x_\WW)$ instead of $\P(X_{\WW} = x_\WW | X_{\SS_\ell} = x_{\SS_\ell})$ and analogously for conditional probabilities. 
Thus we can write $$\avinf ui{\SS_\ell} = 2\cdot \E_{X_{\SS_\ell}}\Big(Q(i+)Q(i-)\big|Q(u+|i+)-Q(u+|i-)\big|\Big)\,.$$

Now for any $i\in \VV\setminus (\SS\cup\{u\})$, 
\begin{align*}
&\sqrt{\tfrac12 \cdot I\big(X_u;X_i|X_{\SS_\ell} \big)}  \\&= \sqrt{\tfrac12 \cdot \sum_{x_{\SS_\ell}}P(x_{\SS_\ell})I\big(X_u;X_i|X_{S_\ell}=x_{\SS_\ell} \big)}
 \stackrel{(a)}{\geq} \sum_{x_{\SS_\ell}}P(x_{\SS_\ell})\sqrt{\tfrac12 \cdot I\big(X_u;X_i|X_{\SS_\ell} = x_{\SS_\ell}\big)} 
 \\& = \E_{X_{\SS_\ell}}\sqrt{\tfrac12\cdot \KL \big( Q(u,i)\,\| \, Q(u)Q(i) \big)}\stackrel{(b)}{\geq}\E_{X_{\SS_\ell}} \dtv (Q(u,i),Q(u) Q(i)) \\
&\stackrel {(c)}{\geq} \E_{X_{\SS_\ell}}\big|Q(u+,i+) - Q(u+)Q(i+) \big|=\E_{X_{\SS_\ell}} \big|Q(u+|i+)Q(i+) - Q(u+)Q(i+) \big| \\
&=\E_{X_{\SS_\ell}}\Big( Q(i+)\cdot\big|Q(u+|i+)\big(Q(i+)+Q(i-)\big) - Q(u+|i+)Q(i+) - Q(u+|i-)Q(i-)\big|\Big)
\\&=\E_{X_{\SS_\ell}}
 \Big(Q(i+)Q(i-)\cdot \big|Q(u+|i+) - Q(u+|i-) \big|\Big)
%\\&\stackrel{(d)}{\geq} \delta^2 \cdot \E_{X_{\SS_\ell}}\big|q(u+|i+) - q(u+|i-) \big| 
\\&= \tfrac12\cdot \avinf u i{\SS_\ell} \\&\stackrel{(e)}{\geq}
\tfrac12\cdot (\eavinf u i{\SS_\ell} -\eps)
 \,.
\end{align*}
(a) is by Jensen's inequality applied to the (concave) square root function, (b) Pinsker's inequality, (c) the definition of total variation distance, (d) Lemma~\ref{l:condRandomness},  (e) is by definition of $\atv(\ell,\eps)$ and the fact that the conditioning set has cardinality $|\SS_\ell|\leq \ell$.

Finally, by definition of the algorithm, node $j_{\ell+1}$ is only added to $\SS_\ell$ if $|\eavinf u{j_{\ell+1}}{\SS_\ell}|\geq \tau$, so the previous displayed equation implies that
$I(X_u;X_{j_{\ell+1}}|X_{\SS_\ell})\geq \tfrac12 (\tau-\eps)^2$
as claimed. 
\end{proof}

\subsection{Key structural result and algorithm correctness}\label{ss:KeyStruct}
We now state our structural result, and use it to prove correctness of the algorithm. Its proof is given in Section~\ref{s:PropProof}. In this subsection we use the values
$
\tau^*,\eps^*,\ell^*
$ defined in Theorem~\ref{t:mainFull}.
\begin{proposition}
\label{p:mutualInf}
Let $G$ be a graph of maximum degree $d$, and consider an Ising model~\eqref{e:Ising} on $G$ with vector of parameters $\theta\in \Omega_{\alpha,\beta,h}(G)$. 
For any node $u\in \VV$, if $\SS\subseteq \VV\setminus \{u\}$ such that $\partial u\not\subseteq \SS$, then there exists a node $i\in \partial u\setminus \SS$ with 
$\avinf ui\SS 
\geq2\tau^*$.
\end{proposition}
We now show that the pseudo-neighborhood contains the true neighborhood. 
\begin{corollary} \label{c:contains}
If event $\atv(\ell^*, \eps^*)$ holds, then for any $u\in \VV$, the pseudo-neighborhood $\SS$ constructed by \textsc{LearnNbhd}$(X^{1:n},\tau^*,u)$ contains the true neighborhood $\partial u$. 
\end{corollary}
\begin{proof}
Consider an arbitrary node $u\in \VV$ and suppose $\partial u\not \subseteq \SS$. Proposition~\ref{p:mutualInf} shows that $\avinf ui\SS\geq 2\tau^*$ for some $i\in \partial u\setminus \SS$. If event $\atv(\ell^*,\eps^*)$ holds, then $|\SS|\leq \ell^*$ by Lemma~\ref{l:PseudoSize}, and we have $\eavinf ui\SS\geq \avinf ui\SS-\eps^*\geq 3\tau^*/2$. But this contradicts line~3 of algorithm \textsc{LearnNbhd}.
\end{proof}

\begin{corollary} Consider the same setup as Corollary~\ref{c:contains}.
After the pruning step, $\SS=\partial u$. 
\end{corollary}
\begin{proof}
By Corollary~\ref{c:contains}, the pseudo-neighborhood $\SS$ contains $\partial u$, hence Equation~\eqref{e:NonNbr} states that $\avinf ui{\SS\setminus\{i\}}=0$ for non-neighbors $i$. By Lemma~\ref{l:PseudoSize}, $|S|\leq \ell^*$, and 
by definition of the event $\atv(\ell^*,\eps^*)$ (with our choice $\eps^*=\tau^*/2$), we have $\eavinf ui{\SS\setminus\{i\}}\leq\eps^*= \tau^*/2$ for all non-neighbors $i$, and hence these are discarded.  Conversely, by Proposition~\ref{p:mutualInf}, $\eavinf ui{\SS\setminus \{i\}}\geq 3\tau^*/2$, so no neighbors are discarded. 
\end{proof}

All the ingredients are in place to finish the proof of Theorem~\ref{t:mainFull}.

\begin{proof}[Proof of Theorem~\ref{t:mainFull}]
By Lemma~\ref{l:closeTV}, our choice of $n$ in the statement of the theorem guarantees $\P(\atv(\ell^*,\eps^*))\geq 1-\zeta$.
Together the two corollaries prove correctness of the algorithm assuming event $\atv(\ell^*,\eps^*)$ holds, and
this completes the proof of Theorem~\ref{t:mainFull}, modulo Proposition~\ref{p:mutualInf}. 
%The proof of the proposition is contained in the following two subsections. 
\end{proof}
%\subsection{Correctness}
%\begin{lemma}
%The conditional mutual information of each neighbor is at least ??? at every step of the algorithm. 
%\end{lemma}

\section{Proof of Proposition~\ref{p:mutualInf}}\label{s:PropProof}

\newcommand{\UU}{{\mathcal{U}}}
%\begin{proof}[Proof of Proposition~\ref{p:mutualInf}]
Fix $u\in \VV$ and $\SS\subseteq \VV\setminus \{u\}$, and let $\UU = \partial u \setminus \SS$ consist of the neighbors of $u$ not in $\SS$. Assume that $\UU$ is nonempty or there is nothing to prove. Let ${\theta}_{u\UU}:=(\theta_{ui})_{i\in \UU}$ and recall the definition
$
\tau^*:= \al^2 \delta^{4d+1}/16d\beta\,.
$ 
We will prove that for any assignment $x_\SS\in \{-,+\}^\SS$, 
%\begin{equation}
%\sum_{i\in \UU}\omega_i\lam_i(x_\SS) \ctv_{u|i;x_\SS}
%\geq \tau^*\,.
%\end{equation}
\begin{equation}\label{e:avgavg}
\sum_{i\in \UU}\theta_{ui}\cdot  \lam_i(x_\SS)\ctv_{u|i;x_\SS} \geq \|{\theta}_{u\UU}\|_1 \cdot 2\tau^*\,.
\end{equation}
Averaging with respect to $x_\SS$ and applying the triangle inequality gives
$$
\sum_{i\in \UU}\theta_{ui}\cdot  \avinf ui\SS
=
\sum_{i\in \UU}\theta_{ui}\cdot  \E\big(\lam_i(X_\SS)|\ctv_{u|i;X_\SS}|\big)
\geq  \|{\theta}_{u\UU}\|_1 \cdot 2\tau^*\,.
$$
As a consequence there exists some $i\in \UU$ such that
$$
\avinf ui\SS
%\geq\frac{|\theta_{ui}|}{ \|{\theta}_{u\UU}\|_1} \cdot \avinf ui\SS 
\geq 2\tau^*\,,
$$
and this proves the Proposition.

\newcommand{\tuti}{\widetilde{\theta_u}}

We proceed with showing \eqref{e:avgavg}.
Let 
$\S = \partial u\cap \SS$ consist of nodes in $\SS$ adjacent to $u$ and $\tuti = \theta_u + \sum_{j\in \S} \theta_{uj} x_j$ be the effective external field at $u$ when we include the effect due to $x_{\S}$. 
Using the notation $Q(\cdot)$ for the conditional measure $\P(\,\cdot\,|X_\SS=x_\SS)$, and $Q(u+|x_\UU) = \P(X_u=+|X_\UU=x_\UU, X_\SS=x_\SS)$, we let
\begin{equation}
g(x_\UU) := Q(u+|x_\UU) = 
\frac{\exp\{2(\theta_u + \sum_{j\in \UU} \theta_{uj}x_j + \sum_{j\in \S} \theta_{uj} x_j)\}}{1+\exp\{2( \theta_u+\sum_{j\in \UU} \theta_{uj}x_j  + \sum_{j\in \S} \theta_{uj} x_j)\}} =\frac{\exp\{2(\tuti + \sum_{j\in \UU} \theta_{uj}x_j) \}}{1+\exp\{2( \tuti+\sum_{j\in \UU} \theta_{uj}x_j )\}}  \,.\label{e:gDef}
\end{equation}
Suppose $i\in \UU$. Conditioning on the values of the remaining neighbors $\UU\setminus \{i\}$ of $u$,
\begin{align}\nonumber
\ctv_{u|i;x_\SS}&=
Q(u+|i+) - Q(u+|i-)  \\&= \nonumber
\sum_{x_{\UU\setminus \{i\}}} \bigg(
Q(u+|i+,x_{\UU\setminus \{i\}}) Q(x_{\UU\setminus \{i\}}|i+)
-Q(u+|i-,x_{\UU\setminus \{i\}}) Q(x_{\UU\setminus \{i\}}|i-)\bigg)
\\ 
&=
\sum_{x_{\UU}}\bigg(
Q(u+|x_{\UU}) Q(x_{\UU})\frac{\ind_{\{x_i=+\}}}{Q(i+)}
-Q(u+|x_{\UU}) Q(x_{\UU})\frac{\ind_{\{x_i=-\}}}{Q(i-)}\bigg)
\nonumber
\\&=\E \bigg(g(Y) \frac{Y_i}{Q(Y_i)}\bigg)\,. \label{e:fxqSum}
\end{align}
The expectation is over the random vector $Y\in \{-,+\}^{|\UU|}$ with law $\P(Y\in \cdot) = \P(X_\UU\in \cdot |X_\SS=x_\SS)$, and in particular $\P(Y_i=+) = \P(X_i=+|X_\SS=x_\SS)=Q(i+)$. The notation $Q(Y_i)$ is understood to mean $Q(i+)$ if $Y_i=+$ and $Q(i-)$ if $Y_i=-$.

It is helpful to rescale and shift the variable $Y_i / Q(Y_i)$ in \eqref{e:fxqSum} so that it takes values $\pm 1$. To this end, define the quantities
$$
s_i:=
\frac12 \cdot \left( \frac1{Q(i+)}- \frac1{1-Q(i+)}\right)\quad \text{and} \quad t_i:=\frac1{\lam_i(x_\SS)}=\frac12 \cdot \left( \frac1{Q(i+)}+ \frac1{1-Q(i+)}\right)\,.
$$
(The function $\lam_i(x_\SS)$ was defined at the beginning of Section~\ref{ss:influence}.)
Arithmetic manipulations lead to
$$
\E Y_i = 2Q(i+) - 1 = - \frac{s_i}{t_i} \qquad \text{and} \qquad
Y_i = \frac{Y_i}{t_i Q(Y_i)} - \frac{s_i}{t_i} = \frac{Y_i}{t_i Q(Y_i)} + \E Y_i
\,.
$$
Multiplying~\eqref{e:fxqSum} by $2\theta_{ui} / t_i$ and using the identities in the last display gives 
\begin{align}
2\frac{\theta_{ui}}{t_i}\cdot \ctv_{u|i;x_\SS} =
 \E  \Big(2g(Y) \frac{Y_i\theta_{ui}}{t_i q(y_i)}\Big)
= 
\E \Big(2g(Y) \big(Y_i\theta_{ui} - 
 \E Y_i \theta_{ui}\big)
 \Big)\nonumber
 =
\E \Big((2g(Y) -1)(Y_i \theta_{ui} - 
   \E Y_i \theta_{ui})\Big)\,. \nonumber
\end{align}  
Summing the last displayed quantity over $i\in \UU$ gives
\begin{equation}
\label{e:temp3}
\sum_{i\in \UU}\theta_{ui}\cdot \lam_i(x_\SS)\ctv_{u|i;x_\SS} =\tfrac12\cdot
\E \Big((2g(Y) -1)(Y\cdot \theta_{u\UU} - 
   \E Y\cdot \theta_{u\UU})\Big)\geq d\beta 2\tau^* \geq \|\theta_{u\UU}\|_1\cdot 2\tau^*
\,.
\end{equation}
Here $\theta_{u\UU} = (\theta_{ui})_{i\in \UU}$ and the first inequality is by 
Lemma~\ref{l:PosCorr}, given in Section~\ref{s:techLemma} below.\qed

\section{Technical lemma}
\label{s:techLemma}
The goal of this section is to justify Equation~\eqref{e:temp3}, which appeared in the proof of Proposition~\ref{p:mutualInf}.
We start with an observation that will be used in the proof.
Recall that $Y$ is a random vector equal in distribution to $X_\UU$ conditioned on $X_\SS=x_\SS$.
Due to the ``conditional randomness" Lemma~\ref{l:condRandomness}, 
$Y$ has probability at least $\delta^{|\UU|}$ of taking each value in $\{-,+\}^{|\UU|}$, where $\delta=\frac12\exp^{-2(\beta d+h)}$. Hence we can decompose the probability mass function $P_Y$ of $Y$ as
\begin{equation}\label{e:xUdecomp}
P_Y(y) = \delta^{|\UU|} + \overbar{P}_Y (y)
\end{equation}
with $\overbar P_Y (y)\geq 0$ for all $y\in \{-,+\}^{|\UU|}$. We will be concerned with the random variable 
\begin{equation}\label{e:Zdef}
 Z:=Y\cdot  \theta_{u\UU} + \tuti\,,
\end{equation} and the decomposition~\eqref{e:xUdecomp} will allow us to obtain anti-concentration for $Z$ from anti-concentration for sums of i.i.d. uniform $\pm1$ random variables. 

The following result of Erd\"os on the Littlewood-Offord problem shows anti-concentration for weighted sums of i.i.d. uniform $\pm 1$ random variables.  
(It can be found, e.g., as Corollary~7.4 in~\cite{tao2006additive} and is a simple consequence of Sperner's Lemma).
\begin{lemma}[Erd\"os \cite{erdos1945lemma}]\label{l:LittlewoodOfford} 
Let $w_1,\dots, w_r$ be real numbers with $|w_i|\geq \alpha$ for all $i$. Let $I=\{t\in \R : t_0-\alpha<t<t_0+\alpha\} $ be an open interval of length $2\alpha$. If $\xi=(\xi_1,\dots,\xi_r)$ is uniformly distributed on $\{-1,1\}^r$, then
$$
\P(w\cdot \xi \in I)\leq \frac{1}{2^r} \cdot {r\choose \lfloor \frac r2 \rfloor}\leq \frac12\,.
$$
\end{lemma}

We can use the decomposition~\eqref{e:xUdecomp} and Lemma~\ref{l:LittlewoodOfford} to draw the following conclusion. (It is possible to show this directly, but this approach seems clearer.) 
We mention that the only place the lower bound $\alpha$ on the coupling strengths appears is through the following lemma.

\begin{lemma}\label{l:zProbBound} Let $2t=2^{|\UU|}\cdot  \delta^{|\UU|}$.
Consider the random variable $Z$ defined in~\eqref{e:Zdef}, and let $\mu= \E Z$. Then $$\E[(\mu-Z)\ind_{\{Z\leq \mu-\al t \}}]\geq  \frac{\alpha t^2}2\,.$$ 

\end{lemma}

\begin{proof}  
Decompose the probability mass function $P_Y$ as discussed at the beginning of this section in~\eqref{e:xUdecomp}. Let $2t=2^{|\UU|}\cdot  \delta^{|\UU|}$ be the total mass assigned to the uniform part.  Let $ Z' = \xi\cdot \theta_{u\UU}+\tuti$ where $\xi\in \{-,+\}^{|\UU|}$ is uniformly distributed and let $Z''=Y''\cdot \theta_{u\UU}+\tuti$ where $Y''\sim (1-M)\inv \overbar P_Y $. The variable $Z$ can be represented as a mixture distribution: 
if we define $R\sim \mbox{Ber}(2t)$, then
$$
Z \stackrel{d}{=} R \cdot Z' + (1-R) \cdot Z''\,,
$$
and we can think of obtaining $Z$ by choosing either $Z'$ or $Z''$ with probabilities $2t$ or $(1-2t)$. Let $I_t=(\mu-t \alpha,\mu+ (2-t)\alpha)$. Lemma~\ref{l:LittlewoodOfford} implies that $\P(Z'\notin I_t)\geq 1/2$ for any $t$, and hence $\P(Z\notin I_t)\geq t$. Denote the probability that $Z$ lies to the left of $I_t$ and inside $I_t$, respectively, as 
$$m_1 = \P(Z\leq \mu - t \al)\qquad \text{and}\qquad m_2 = \P(Z\in I_t)\leq 1- t\,,$$ so that $1-m_1-m_2 = \P(Z\geq \mu+ (2-t)\al)$. Thinking about placing the probability mass to minimize $\mu$ subject to fixed $m_1$ and $m_2$ justifies the inequality
\begin{align*}
\mu = \E Z &= m_1\cdot \E(Z\mid Z\leq \mu-t\al ) + m_2\cdot  \E(Z\mid Z\in I_t) + (1-m_1-m_2) \E(Z\mid Z\geq \mu+ (2-t)\al)
\\&\geq m_1\cdot \E(Z \mid Z\leq \mu-t\al ) +m_2(\mu-t\al)+(1-m_1-m_2)(\mu+ (2-t)\al) 
\,.
\end{align*}
Using $m_2\leq 1-t$ and performing arithmetic manipulations leads to 
$$
m_1\cdot \E[(\mu-\al t) - Z \mid Z\leq \mu-\al t] + m_1 2\al\geq -m_2 2\alpha + (2-t)\al\geq t\al\,.
$$
At least one of the two terms on the left-hand side is larger than the average, which is at least $ t\al/2$, and in either case $\E[(\mu-Z)\ind_{\{Z\leq \mu-\al t \}}]\geq  \al t^2 /2$ (using the fact that $t\leq 1$).
\end{proof}

The remainder of this section is devoted to proving the following lemma.

\begin{lemma}\label{l:PosCorr} Let $g$ and $Y$ be as in the proof of Proposition~\ref{p:mutualInf} in Section~\ref{s:PropProof}. Then the quantity in Equation~\eqref{e:temp3} is lower bounded as
\begin{align*}\tfrac12\cdot \E\Big((2g(Y) -1)(Y\cdot \theta_{u\UU} - 
   \E  Y\cdot \theta_{u\UU})\Big)&\geq  \frac{\al^2 \delta^{4|\UU|+1}}{8} \geq d\beta 2\tau^*
   \,.\end{align*}
\end{lemma}
\begin{proof}
We start by adding and subtracting $\tuti$ to the left-hand side of the lemma statement: 
\begin{align*}
&\E \Big(\big(2g(Y) -1\big)\big(Y\cdot \theta_{u\UU} - 
   \E Y\cdot \theta_{u\UU}\big)\Big)=
 \E\Big(\big(2g(Y) -1\big)  \big(Y\cdot \theta_{u\UU} +\tuti  - 
     \E (Y \cdot \theta_{u\UU}+ \tuti)\big)\Big)\,.\nonumber
\end{align*}
Recalling the definition of $g$ in \eqref{e:gDef}, we make the observation that $2g(x) - 1 = \tanh(x \cdot \theta_{u\UU} +\tuti)$. We will use the fact that $\tanh(z)$ is an odd, increasing function, which is concave for $t\geq 0$ and convex for $t\leq 0$. Recall from \eqref{e:Zdef} the definition $Z = Y \cdot \theta_{u\UU} +\tuti$ and let $\mu:=\E Z$.
The lemma statement requires that we lower bound $\E[\tanh(Z)(Z-\mu)].$
We assume from now onward that $\mu\geq 0$, but a (symmetrically) identical argument applies to the opposite case $\mu \leq 0$. From the definition of $Z$ and assumptions $|\theta_{ui}|\leq \beta$, $|\theta_u|\leq h$ we obtain the bound
$$\mu = \E Z  \leq (|\UU| + |\widetilde \SS|)\cdot \beta + h\leq \beta d+h\,.$$

Next we record a few estimates on the function $\tanh(\cdot)$.
The derivative satisfies $$\frac{d}{dz}\tanh(z) = \frac{4}{(e^z+{e^{-z}})^2}\geq \frac{1}{e^{2|z|}} $$
and due to the concavity of $\tanh(z)$ for $z\geq0$, we have the estimate
$$
\tanh(z)\leq \tanh(\mu) - \frac{ (\mu-z)}{e^{2\mu}}\leq 
\tanh(\mu) - \frac{ (\mu-z)}{e^{2(\beta d + h)}} = \tanh(\mu)  - 2\delta (\mu - z)
\qquad  \text{for } 0\leq z\leq \mu\,.
$$
We additionally use the bound $\tanh(t)\geq \tanh(\mu)$ for $z\geq \mu$ due to monotonicity of $\tanh(\cdot)$. 
Partitioning the range of $Z$ and using these estimates gives
\begin{align*}
&\E[\tanh(Z)(Z-\mu)] \\&= \E[\tanh(Z)(Z-\mu)\ind_{\{Z<0\}}] +\E[\tanh(Z)(Z-\mu)\ind_{\{Z\in [0,\mu]\}}] + \E[\tanh(Z)(Z-\mu)\ind_{\{Z>\mu\}}] 
\\&
\geq\E[\tanh(Z)(Z-\mu)\ind_{\{Z<0\}}] + 
\E\Big[\Big(\tanh(\mu) - 2\delta (\mu-Z)\Big)(Z-\mu) \ind_{\{Z\in [0,\mu]\}}
\Big]
\\ &\qquad+ \E[\tanh(\mu)(Z-\mu)\ind_{\{Z>\mu\}}]\,.
\end{align*}
Subtracting $\tanh(\mu)\E (Z-\mu) = 0$ from the third term, the last expression is equal to
\begin{align}
&\E[\tanh(Z)(Z-\mu)\ind_{\{Z<0\}}] + 
\E\Big[\Big(\tanh(\mu) - 2\delta (\mu-Z)\Big)(Z-\mu) \ind_{\{Z\in [0,\mu]\}}
\Big]  \nonumber
\\& - \E [ \tanh(\mu)(Z-\mu)\ind_{\{Z<0\}}] - \E [ \tanh(\mu)(Z-\mu)\ind_{\{Z\in [0,\mu]\}}]\nonumber
\\=&
\E[\big(\tanh(Z)-\tanh(\mu)\big)(Z-\mu)\ind_{\{Z<0\}}] + 
\E\Big[ 2\delta(\mu-Z)^2\ind_{\{Z\in [0,\mu]\}}
\Big]\,. \label{e:temp4}
\end{align}
Both of these terms are non-negative.

Lemma~\ref{l:zProbBound} states that
$\E[(\mu-Z)\ind_{\{Z\leq \mu-\al t \}}]\geq  \frac{\alpha t^2}2$, where $2t = (2\delta)^{|\UU|}$, and
this means that either
\begin{equation}\label{e:techLemmaCases}
\E[(\mu-Z)\ind_{\{Z\leq \mu-\al t \}}\ind_{\{Z<0 \}}]\geq  \frac{\alpha t^2}4
 \qquad \text{or}\qquad 
\E[(\mu-Z)\ind_{\{Z\leq \mu-\al t \}}\ind_{\{Z\in[0,\mu] \}}]\geq  \frac{\alpha t^2}4\,,
\end{equation}
(or both) is true. 
In the former case, the first term in \eqref{e:temp4} is lower bounded by
\begin{align}
&\E[\big(\tanh(Z)-\tanh(\mu)\big)(Z-\mu)\ind_{\{Z\leq \mu-\al t \}}\ind_{\{Z<0 \}}] \nonumber
\\ &\geq 
\ind_{\{\mu<1\}}\frac{\alpha t^2}{4}\big(\tanh(\mu) - \tanh(\mu-\alpha t)\big)  + 
\ind_{\{\mu\geq1\}}\frac{\alpha t^2}{4} \tanh(1)\label{e:twoTermsTemp}
\\ &\stackrel{(a)}{\geq} 
\frac{\alpha t^2}{4}\min\Big(\frac{\alpha t}{e^{2}}, \tanh(1) \Big)  
\stackrel{(b)}{\geq} \frac{\al^2t^3}{4 e^{2}}
\geq \frac{\al^2 t^3}{2^5}
\label{e:temp5} \,.
\end{align} 
(a) follows by bounding the first term in \eqref{e:twoTermsTemp} for $0\leq \mu<1$ by noting that $[\mu-\al t, \mu]\subseteq [-\al t,1]\subseteq [-1,1]$ and lower bounding the derivative of $\tanh (z)$ on $[-1,1]$ by $1/e^2$, and
(b) follows from the fact that $\tanh(1)> 1/e^2\geq \al t / e^2$.

In the latter case of \eqref{e:techLemmaCases}, the second term of \eqref{e:temp4} is lower bounded by
\begin{equation}\label{e:temp6}
2\delta\cdot \E\Big( (\mu-Z)^2 \ind_{\{Z\leq \mu-\al t \}}\ind_{\{Z\in[0,\mu] \}}
\Big)
\geq 
\frac{2\delta\al^2t^4}{16}\,.
\end{equation}
This used Cauchy-Schwarz (or equivalently non-negativity of the variance) to lower bound the expectation of $(\mu-Z)^2$ by $(\al t^2/4)^2$, the square of the expectation given in \eqref{e:techLemmaCases}.
The quantity in \eqref{e:temp6} is smaller than in \eqref{e:temp5}, because $2\delta\leq 1$ and $t\leq 1/2$. Multiplying the right-hand side of \eqref{e:temp6} by $1/2$ and plugging in $t = \frac12(2\delta)^{|\UU|}\geq \delta^{|\UU|}$  completes the proof of the lemma.
\end{proof}

\section{Proof of Lemma~\ref{l:closeTV}}
\label{s:EmpiricalLemma}
	Azuma's inequality states that if $Y\sim
	\hbox{Bin}(n,\mu)$, then
	\[
	P(|Y-n\mu| > \gamma n) \leq 2\exp(-2\gamma^2 n)\,,
	\]
	so for any subset of nodes $\WW \subseteq \VV$ and configuration
	$x_\WW \in \{-,+\}^{|\WW|}$ we have
	\begin{equation}\label{e:chernoff}
	\P\Big( \big| \Ph (X_\WW=x_\WW) - \P(X_\WW=x_\WW) \big| \geq
	\gamma \Big) \leq 2\exp(-2\gamma^2 n).
	\end{equation}
	There are $2^{|\WW|} {p\choose |\WW|}\leq (2p)^{|\WW|}$ such choices of
	$\WW$ and $x_\WW$ of a given cardinality, and hence at most $(\ell+2)(2p)^{\ell+2}$ choices of $\WW$ and $x_\WW$ with $|\WW|\leq \ell+2$.  
	
	Suppose $n \geq (2\gamma^2)^{-1} \log\big(2(\ell+2)(2p)^{\ell+2}/\zeta\big)$. 
	An application of the
	union bound implies that with probability at least $$1-(\ell+2)(2p)^{\ell+2}\cdot
	2\exp(-2\gamma^2 n)\geq  1-\zeta
	$$ it holds that
	\begin{equation}\label{e:probabilityBoundNew}
	\big| \Ph (X_\WW=x_\WW) - \P(X_\WW=x_\WW) \big| \leq \gamma
	\end{equation}
	for all $\WW$ and $x_\WW$ with $|\WW|\leq \ell+2$.
	For the remainder of the proof assume \eqref{e:probabilityBoundNew} holds.
	
	Our goal is to bound the quantity
		$$
		\big| \avinf ui\SS -  \eavinf ui\SS\big|=
		\big|\E_{X_\SS\sim \P} \big(\lam_i(X_\SS)|\ctv_{u|i;X_\SS}|\big) - \E_{X_\SS\sim\Ph} \big(\elam_i(X_\SS)|\empctv_{u|i;X_\SS}|\big) \big|\,.
		$$
	The triangle inequality and the inequality $\big||s|-|t|\big|\leq |s-t|$ for real-valued $s$ and $t$ gives
	\begin{align*}
	&\big|\E_{X_\SS\sim \P} \big(\lam_i(X_\SS)|\ctv_{u|i;X_\SS}|\big) - \E_{X_\SS\sim\Ph} \big(\elam_i(X_\SS)|\empctv_{u|i;X_\SS}|\big) \big|
	\\&=\bigg|\sum_{x_\SS}\Big[\P(X_\SS=x_\SS) \lam_i(x_\SS)|\ctv_{u|i;x_\SS}| - \Ph(X_\SS=x_\SS) \elam_i(x_\SS)|\empctv_{u|i;x_\SS}|\Big] \bigg|
	\\&\leq 
	\sum_{x_\SS}\Big|\P(X_\SS=x_\SS) \lam_i(x_\SS)|\ctv_{u|i;x_\SS}| - \Ph(X_\SS=x_\SS) \elam_i(x_\SS)|\empctv_{u|i;x_\SS}| \Big|
	\\&\leq
	\sum_{x_\SS}\Big|\P(X_\SS=x_\SS) \lam_i(x_\SS)\ctv_{u|i;x_\SS} - \Ph(X_\SS=x_\SS) \elam_i(x_\SS)\empctv_{u|i;x_\SS}\Big| \,.
	\end{align*}
	Writing out the definition of $\ctv_{u|i;X_\SS}$ and $\empctv_{u|i;X_\SS}$, the above sum is equal to
	\begin{align*}
	&
	\sum_{X_\SS}\Big| \P(X_\SS=x_\SS)\lam_i(x_\SS)\big(\P(X_u=+|X_i=+,X_\SS=x_\SS) - \P(X_u=+|X_i=-,X_\SS=x_\SS)\big)
	\\&\qquad-\Ph(X_\SS=x_\SS)\elam_i(x_\SS)\big(\Ph(X_u=+|X_i=+,X_\SS = x_\SS)-\Ph(X_u=+|X_i=-,X_\SS = x_\SS)\big)\Big|
	\\&\stackrel{(a)}{=}
	\sum_{x_\SS}\Bigg|\bigg[\lam_i(x_\SS)\frac{ \P(X_u=+,X_i=+,X_\SS=x_\SS)}{\P(X_i=+|X_\SS=x_\SS)} -\elam_i(x_\SS)\frac{ \Ph(X_u=+,X_i=+,X_\SS=x_\SS)}{\Ph(X_i=+|X_\SS=x_\SS)}\bigg]
	\\& \qquad -\bigg[\lam_i(x_\SS)\frac{ \P(X_u=+,X_i=-,X_\SS=x_\SS)}{\P(X_i=-|X_\SS=x_\SS)} -\elam_i(x_\SS)\frac{ \Ph(X_u=+,X_i=-,X_\SS=x_\SS)}{\Ph(X_i=-|X_\SS=x_\SS)}\bigg]
	\Bigg|
	\\&\stackrel{(b)}{\leq} 
	\sum_{x_\SS}\bigg|\lam_i(x_\SS)\frac{ \P(X_u=+,X_i=+,X_\SS=x_\SS)}{\P(X_i=+|X_\SS=x_\SS)} -\elam_i(x_\SS)\frac{ \Ph(X_u=+,X_i=+,X_\SS=x_\SS)}{\Ph(X_i=+|X_\SS=x_\SS)}\bigg|
	\\& \qquad +\sum_{x_\SS}\bigg|\lam_i(x_\SS)\frac{ \P(X_u=+,X_i=-,X_\SS=x_\SS)}{\P(X_i=-|X_\SS=x_\SS)} -\elam_i(x_\SS)\frac{ \Ph(X_u=+,X_i=-,X_\SS=x_\SS)}{\Ph(X_i=-|X_\SS=x_\SS)}\bigg|
	\\&:= C^++ C^-\,.
	\end{align*}
	Here (a) is by Bayes' rule and (b) is by the triangle inequality.

	We will now bound the quantity $C^+$ in a way that does not depend on the specific assignment of $\pm$ to $X_i$, so the same bound will hold symmetrically for $C^-$. Using the identity $ab - \widehat{a}\widehat{b} = ab - a\widehat{b}+a\widehat{b}-\widehat{a}\widehat{b}
	=a(b - \widehat{b})+\widehat{b}(a-\widehat{a})$, the triangle inequality, and the definition of $\lam,\elam$, we have
	\begin{align}
	C^+&=\sum_{x_\SS}\bigg|\lam_i(x_\SS)
	\frac{ \P(X_u=+,X_i=+,X_\SS=x_\SS)}{\P(X_i=+|X_\SS=x_\SS)} -\elam_i(x_\SS)\frac{ \Ph(X_u=+,X_i=+,X_\SS=x_\SS)}{\Ph(X_i=+|X_\SS=x_\SS)}\bigg| \nonumber
	\\&\leq\sum_{x_\SS}\bigg| \P(X_u=+,X_i=+,X_\SS=x_\SS)\Big(\frac{\lam_i(x_\SS)}{\P(X_i=+|X_\SS=x_\SS)} -\frac{\elam_i(x_\SS)}{\Ph(X_i=+|X_\SS=x_\SS)}\bigg|\nonumber
	\\&\quad + \sum_{x_\SS}\bigg|\frac{\elam_i(x_\SS)}{\Ph(X_i=+|X_\SS=x_\SS)}\Big(
	\P(X_u=+,X_i=+,X_\SS=x_\SS) -  \Ph(X_u=+,X_i=+,X_\SS=x_\SS)\Big)\bigg|\nonumber
	\\&=2\cdot \sum_{x_\SS} \P(X_u=+,X_i=+,X_\SS=x_\SS)\Big|\P(X_i=-|X_\SS=x_\SS) -\Ph(X_i=-|X_\SS=x_\SS)\Big|	\nonumber
	\\&\quad + 2\cdot\sum_{x_\SS}\Ph(X_i=-|X_\SS=x_\SS)\Big|
	\P(X_u=+,X_i=+,X_\SS=x_\SS) -  \Ph(X_u=+,X_i=+,X_\SS=x_\SS)\Big|
\label{e:CplusTemp}
	\end{align}
The latter sum is bounded by $2\cdot 2^{|\SS|}\gamma$. In order to bound the first sum in \eqref{e:CplusTemp} we write
	\begin{align*}
	&\Big|\P(X_i=-|X_\SS=x_\SS) -\Ph(X_i=-|X_\SS=x_\SS)\Big|
	\\&=\Big|\frac{\P(X_i=-,X_\SS=x_\SS)}{\P(X_\SS=x_\SS)} -\frac{\Ph(X_i=-,X_\SS=x_\SS)}{\Ph(X_\SS=x_\SS)}\Big|
	\\&\leq 
	\Big|\frac{\P(X_i=-,X_\SS=x_\SS)}{\P(X_\SS=x_\SS)} -\frac{\Ph(X_i=-,X_\SS=x_\SS)}{\P(X_\SS=x_\SS)}\Big| + \Big|\frac{\Ph(X_i=-,X_\SS=x_\SS)}{\P(X_\SS=x_\SS)} -\frac{\Ph(X_i=-,X_\SS=x_\SS)}{\Ph(X_\SS=x_\SS)}\Big|
	\\&\leq\frac{2\gamma}{q} \,,
	\end{align*}
	where 
	$q:=\delta^{\ell}\leq \delta^{|\SS|}\leq \min_{x_\SS}\P(X_\SS=x_\SS)\,.$
	Plugging this into \eqref{e:CplusTemp}, the sum over $x_\SS$ marginalizes over these variables and we obtain
	\begin{align*}
	C^+\leq 
	\frac{4\gamma}{q}\cdot \P(X_u=+,X_i=+)
	+ 2^{|\SS|+1}\gamma
	\leq \frac{6\gamma}{q} \,.
	\end{align*}
	Here we used the fact that $q\inv = \delta^{-\ell}\geq 2^{|\SS|}$, since $\delta\leq 1/2$. 
	
	The same bound holds for $C^-$, so
	$$
	\big| \avinf uiS -  \eavinf uiS\big|=
	\big|\E_{X_\SS\sim \P} \big(\lam_i(X_\SS)|\ctv_{u|i;X_\SS}|\big) - \E_{X_\SS\sim\Ph} \big(\elam_i(X_\SS)|\empctv_{u|i;X_\SS}|\big) \big|
	\leq \frac{12\gamma}{q}\,.
	$$
	Choosing $\gamma = \eps \delta^{\ell}/12$, we get the desired accuracy and our earlier choice of $n$ evaluates to
	$$
	n = (2\gamma^2)^{-1} \log\bigg(\frac{2(\ell+2)(2p)^{\ell+2}}{\zeta}\bigg)
	\leq \frac{144(\ell+3)}{\eps^2\delta^{2\ell}}\log\frac p \zeta\,.
	$$
	\vskip-1.2\baselineskip
	\qed

\section{Discussion}\label{s:discussion}
Our algorithm learns Ising models in time quadratic in the number of nodes (ignoring the log factor), showing that high graph degree is not an obstacle to efficient structure learning.
In light of Valiant's \cite{valiant2012finding} algorithm for finding large correlations in less than quadratic time, it is plausible that one could achieve an analogous further improvement of the runtime to $p^c$ for some constant $c<2$. Perhaps even ``input-sparsity" time $\OO(pd)$ is possible.
As far as practical applicability, it seems most urgent to improve upon the doubly-exponential dependence of
the run-time and sample complexity on $\beta d$. We suspect that (singly) exponential dependence is possible. 

There are a number of further questions of both theoretical and practical relevance. While Lemma~\ref{l:closeTV} on estimation of influences can be easily modified to deal with mean-zero noise, generalizing to more complicated noise requires more work. Another problem is that of structure learning with unobserved or latent variables, and it would be interesting to see if the approach taken here has implications also for this more challenging variant. Yet another issue is that in practice one does not a priori know the degree $d$ nor the parameter bounds $\al,\beta,h$; perhaps using the Bayesian Information Criterion or other method can help determine model complexity. 

Many generalizations and extensions of the results are possible. For example, one can likely generalize to pairwise Markov random fields with alphabet sizes larger than two, but the required non-degeneracy conditions will be more complicated. Also, the assumption of uniformly bounded node degrees may be restrictive in some settings, and the algorithm presented here can probably be modified to work on families of graphs with unbounded degree (but low average degree).

\section*{Acknowledgements}
I am extremely grateful to David Gamarnik and Devavrat Shah for countless discussions on graphical models and related topics over the last two years. I thank Bruce Hajek for comments on a draft of the paper and Sahand Negahban and Costis Daskalakis for stimulating conversations.

\newpage 

\setlength{\bibsep}{6pt}

{\small
\bibliographystyle{plainnat}

\bibliography{MRF_BIB}

\begin{thebibliography}{67}
\providecommand{\natexlab}[1]{#1}
\providecommand{\url}[1]{\texttt{#1}}
\expandafter\ifx\csname urlstyle\endcsname\relax
  \providecommand{\doi}[1]{doi: #1}\else
  \providecommand{\doi}{doi: \begingroup \urlstyle{rm}\Url}\fi

\bibitem[Abbeel et~al.(2006)Abbeel, Koller, and Ng]{abbeel2006learning}
P.~Abbeel, D.~Koller, and A.~Ng.
\newblock Learning factor graphs in polynomial time and sample complexity.
\newblock \emph{JMLR}, 2006.

\bibitem[Ackley et~al.(1985)Ackley, Hinton, and Sejnowski]{ackley1985learning}
D.~Ackley, G.~Hinton, and T.~Sejnowski.
\newblock A learning algorithm for boltzmann machines*.
\newblock \emph{Cognitive science}, 9\penalty0 (1):\penalty0 147--169, 1985.

\bibitem[Anandkumar et~al.(2012{\natexlab{a}})Anandkumar, Huang, Hsu, and
  Kakade]{anandkumar2012learningB}
A.~Anandkumar, F.~Huang, D.~Hsu, and S.~Kakade.
\newblock Learning mixtures of tree graphical models.
\newblock In \emph{NIPS}, 2012{\natexlab{a}}.

\bibitem[Anandkumar et~al.(2012{\natexlab{b}})Anandkumar, Tan, Huang, and
  Willsky]{anandkumar2012high}
A.~Anandkumar, V.~Tan, F.~Huang, and A.~Willsky.
\newblock High-dimensional structure estimation in {I}sing models: Local
  separation criterion.
\newblock \emph{Annals of Stat.}, 40\penalty0 (3):\penalty0 1346--1375,
  2012{\natexlab{b}}.

\bibitem[Aurell et~al.(2010)Aurell, Ollion, and Roudi]{aurell2010dynamics}
E.~Aurell, C.~Ollion, and Y.~Roudi.
\newblock Dynamics and performance of susceptibility propagation on synthetic
  data.
\newblock \emph{The European Physical Journal B-Condensed Matter and Complex
  Systems}, 77\penalty0 (4):\penalty0 587--595, 2010.

\bibitem[Bandyopadhyay and Gamarnik(2008)]{bandyopadhyay2008counting}
A.~Bandyopadhyay and D.~Gamarnik.
\newblock Counting without sampling: Asymptotics of the log-partition function
  for certain statistical physics models.
\newblock \emph{Random Structures \& Algorithms}, 33\penalty0 (4):\penalty0
  452--479, 2008.

\bibitem[Belkin and Sinha(2010)]{belkin2010polynomial}
M.~Belkin and K.~Sinha.
\newblock Polynomial learning of distribution families.
\newblock In \emph{FOCS}, pages 103--112, 2010.

\bibitem[Bento and Montanari(2009)]{bento2009graphical}
J.~Bento and A.~Montanari.
\newblock Which graphical models are difficult to learn?
\newblock In \emph{NIPS}, 2009.

\bibitem[Bresler et~al.(2008)Bresler, Mossel, and Sly]{BMS08}
G.~Bresler, E.~Mossel, and A.~Sly.
\newblock Reconstruction of {M}arkov random fields from samples: Some
  observations and algorithms.
\newblock In \emph{APPROX}, 2008.

\bibitem[Bresler et~al.(2014{\natexlab{a}})Bresler, Gamarnik, and Shah]{BGS14a}
G.~Bresler, D.~Gamarnik, and D.~Shah.
\newblock Structure learning of antiferromagnetic {I}sing models.
\newblock In \emph{NIPS}, 2014{\natexlab{a}}.

\bibitem[Bresler et~al.(2014{\natexlab{b}})Bresler, Gamarnik, and Shah]{BGS14b}
G.~Bresler, D.~Gamarnik, and D.~Shah.
\newblock Hardness of parameter estimation in graphical models.
\newblock In \emph{NIPS}, 2014{\natexlab{b}}.

\bibitem[Brush({1967})]{RevModPhys}
S.~Brush.
\newblock {History of the Lenz-Ising Model}.
\newblock \emph{Rev. Mod. Phys.}, 39\penalty0 (4):\penalty0 883--893, Oct
  {1967}.

\bibitem[Chow and Liu(1968)]{chow1968approximating}
C.~Chow and C.~Liu.
\newblock Approximating discrete probability distributions with dependence
  trees.
\newblock \emph{IEEE Trans. on Info. Theory}, 14\penalty0 (3):\penalty0
  462--467, 1968.

\bibitem[Cocco and Monasson(2012)]{cocco2012adaptive}
S.~Cocco and R.~Monasson.
\newblock Adaptive cluster expansion for the inverse {I}sing problem:
  convergence, algorithm and tests.
\newblock \emph{Journal of Statistical Physics}, pages 1--63, 2012.

\bibitem[Cocco et~al.(2009)Cocco, Leibler, and Monasson]{cocco2009neuronal}
S.~Cocco, S.~Leibler, and R.~Monasson.
\newblock Neuronal couplings between retinal ganglion cells inferred by
  efficient inverse statistical physics methods.
\newblock \emph{Proceedings of the National Academy of Sciences}, 106\penalty0
  (33):\penalty0 14058--14062, 2009.

\bibitem[Cover and Thomas(2012)]{cover2012elements}
T.~Cover and J.~Thomas.
\newblock \emph{Elements of information theory}.
\newblock John Wiley \& Sons, 2012.

\bibitem[Csisz{\'a}r and Talata(2006)]{csiszar2006consistent}
I.~Csisz{\'a}r and Z.~Talata.
\newblock Consistent estimation of the basic neighborhood of {M}arkov random
  fields.
\newblock \emph{Annals of Stat.}, pages 123--145, 2006.

\bibitem[Dasgupta(1999)]{dasgupta1999learning}
S.~Dasgupta.
\newblock Learning polytrees.
\newblock In \emph{UAI}, 1999.

\bibitem[Decelle and Ricci-Tersenghi(2014)]{decelle2014pseudolikelihood}
A.~Decelle and F.~Ricci-Tersenghi.
\newblock Pseudolikelihood decimation algorithm improving the inference of the
  interaction network in a general class of ising models.
\newblock \emph{Physical review letters}, 112\penalty0 (7):\penalty0 070603,
  2014.

\bibitem[Dobrushin(1970)]{dobrushin1970prescribing}
R.~Dobrushin.
\newblock Prescribing a system of random variables by conditional
  distributions.
\newblock \emph{Theory of Probability \& Its Applications}, 15\penalty0
  (3):\penalty0 458--486, 1970.

\bibitem[Dobrushin and Shlosman(1985)]{dobrushin1985constructive}
R.~Dobrushin and S.~Shlosman.
\newblock Constructive criterion for the uniqueness of {G}ibbs field.
\newblock In \emph{Statistical physics and dynamical systems}, pages 347--370.
  Springer, 1985.

\bibitem[Dyer et~al.(2004)Dyer, Sinclair, Vigoda, and Weitz]{dyer2004mixing}
M.~Dyer, A.~Sinclair, E.~Vigoda, and D.~Weitz.
\newblock Mixing in time and space for lattice spin systems: A combinatorial
  view.
\newblock \emph{Random Structures \& Algorithms}, 24\penalty0 (4):\penalty0
  461--479, 2004.

\bibitem[Erd{\"o}s(1945)]{erdos1945lemma}
P.~Erd{\"o}s.
\newblock On a lemma of {L}ittlewood and {O}fford.
\newblock \emph{Bulletin of the American Mathematical Society}, 51\penalty0
  (12):\penalty0 898--902, 1945.

\bibitem[Friedman et~al.(2008)Friedman, Hastie, and
  Tibshirani]{friedman2008sparse}
J.~Friedman, T.~Hastie, and R.~Tibshirani.
\newblock Sparse inverse covariance estimation with the graphical lasso.
\newblock \emph{Biostatistics}, 9\penalty0 (3):\penalty0 432--441, 2008.

\bibitem[Gamarnik(2013)]{gamarnikcorrelation}
D.~Gamarnik.
\newblock Correlation decay method for decision, optimization, and inference in
  large-scale networks.
\newblock \emph{Tutorials in Operations Research, INFORMS}, 2013.

\bibitem[Gamarnik and Katz(2007)]{gamarnik2007correlation}
D.~Gamarnik and D.~Katz.
\newblock Correlation decay and deterministic fptas for counting list-colorings
  of a graph.
\newblock In \emph{SODA}, pages 1245--1254, 2007.

\bibitem[Goldberg et~al.(2003)Goldberg, Jerrum, and
  Paterson]{goldberg2003computational}
L.~Goldberg, M.~Jerrum, and M.~Paterson.
\newblock The computational complexity of two-state spin systems.
\newblock \emph{Random Structures \& Algorithms}, 23\penalty0 (2):\penalty0
  133--154, 2003.

\bibitem[Hinton and Sejnowski(1986)]{hinton1986learning}
G.~Hinton and T.~Sejnowski.
\newblock Learning and relearning in {B}oltzmann machines.
\newblock \emph{MIT Press}, 1\penalty0 (282-317):\penalty0 4--2, 1986.

\bibitem[Ising(1925)]{ising1925beitrag}
E.~Ising.
\newblock Beitrag zur theorie des ferromagnetismus.
\newblock \emph{Zeitschrift f{\"u}r Physik A Hadrons and Nuclei}, 31\penalty0
  (1):\penalty0 253--258, 1925.

\bibitem[Jalali et~al.(2011{\natexlab{a}})Jalali, Johnson, and
  Ravikumar]{jalali2011learning}
A.~Jalali, C.~Johnson, and P.~Ravikumar.
\newblock On learning discrete graphical models using greedy methods.
\newblock \emph{arXiv preprint arXiv:1107.3258}, 2011{\natexlab{a}}.

\bibitem[Jalali et~al.(2011{\natexlab{b}})Jalali, Ravikumar, Vasuki, and
  Sanghavi]{jalali2011learning2}
A.~Jalali, P.~Ravikumar, V.~Vasuki, and S.~Sanghavi.
\newblock On learning discrete graphical models using group-sparse
  regularization.
\newblock In \emph{AISTATS}, volume~14, 2011{\natexlab{b}}.

\bibitem[Jerrum and Sinclair(1993)]{jerrum1993polynomial}
M.~Jerrum and A.~Sinclair.
\newblock Polynomial-time approximation algorithms for the {I}sing model.
\newblock \emph{SIAM Journal on computing}, 22\penalty0 (5):\penalty0
  1087--1116, 1993.

\bibitem[Kalai et~al.(2009)Kalai, Samorodnitsky, and Teng]{kalai2009learning}
A.~Kalai, A.~Samorodnitsky, and S.H. Teng.
\newblock Learning and smoothed analysis.
\newblock In \emph{FOCS}, pages 395--404. IEEE, 2009.

\bibitem[Kane et~al.(2013)Kane, Klivans, and Meka]{kane2013learning}
D.~Kane, A.~Klivans, and R.~Meka.
\newblock Learning halfspaces under log-concave densities: Polynomial
  approximations and moment matching.
\newblock In \emph{Conference on Learning Theory}, pages 522--545, 2013.

\bibitem[Klivans and Meka(2013)]{klivans2013moment}
A.~Klivans and R.~Meka.
\newblock Moment-matching polynomials.
\newblock \emph{arXiv preprint arXiv:1301.0820}, 2013.

\bibitem[Lauritzen(1996)]{lauritzen1996graphical}
S.~Lauritzen.
\newblock \emph{Graphical models}.
\newblock Oxford University Press, 1996.

\bibitem[Lee et~al.(2006)Lee, Ganapathi, and Koller]{lee2006efficient}
S.~Lee, V.~Ganapathi, and D.~Koller.
\newblock Efficient structure learning of {M}arkov networks using $\ell_1
  $-regularization.
\newblock In \emph{NIPS}, pages 817--824, 2006.

\bibitem[Lezon et~al.(2006)Lezon, Banavar, Cieplak, Maritan, and
  Fedoroff]{lezon2006using}
T.~Lezon, J.~Banavar, M.~Cieplak, A.~Maritan, and N.~Fedoroff.
\newblock Using the principle of entropy maximization to infer genetic
  interaction networks from gene expression patterns.
\newblock \emph{Proceedings of the National Academy of Sciences}, 103\penalty0
  (50):\penalty0 19033--19038, 2006.

\bibitem[Martinelli and Olivieri(1994)]{martinelli1994approach}
F.~Martinelli and E.~Olivieri.
\newblock Approach to equilibrium of {G}lauber dynamics in the one phase
  region.
\newblock \emph{Comm. in Mathematical Physics}, 161\penalty0 (3):\penalty0
  447--486, 1994.

\bibitem[Meinshausen and B\"uhlmann(2006)]{meinshausen2006high}
N.~Meinshausen and P.~B\"uhlmann.
\newblock High-dimensional graphs and variable selection with the lasso.
\newblock \emph{The Annals of Statistics}, pages 1436--1462, 2006.

\bibitem[M{\'e}zard and Mora(2009)]{mezard2009constraint}
M.~M{\'e}zard and T.~Mora.
\newblock Constraint satisfaction problems and neural networks: A statistical
  physics perspective.
\newblock \emph{Journal of Physiology-Paris}, 103\penalty0 (1):\penalty0
  107--113, 2009.

\bibitem[Moitra and Valiant(2010)]{moitra2010settling}
A.~Moitra and G.~Valiant.
\newblock Settling the polynomial learnability of mixtures of {G}aussians.
\newblock In \emph{FOCS}, pages 93--102, 2010.

\bibitem[{Montanari}(2014)]{Montanari2014}
A.~{Montanari}.
\newblock {Computational Implications of Reducing Data to Sufficient
  Statistics}.
\newblock \emph{ArXiv e-prints}, September 2014.

\bibitem[Mora et~al.(2010)Mora, Walczak, Bialek, and Callan]{mora2010maximum}
T.~Mora, A.~Walczak, W.~Bialek, and C.~Callan.
\newblock Maximum entropy models for antibody diversity.
\newblock \emph{Proceedings of the National Academy of Sciences}, 107\penalty0
  (12):\penalty0 5405--5410, 2010.

\bibitem[Negahban et~al.(2012)Negahban, Ravikumar, Wainwright, and
  Yu]{negahban2012unified}
S.~Negahban, P.~Ravikumar, M.~Wainwright, and B.~Yu.
\newblock A unified framework for high-dimensional analysis of {M}-estimators
  with decomposable regularizers.
\newblock \emph{Statistical Science}, 27\penalty0 (4):\penalty0 538--557, 2012.

\bibitem[Netrapalli et~al.(2010)Netrapalli, Banerjee, Sanghavi, and
  Shakkottai]{netrapalli2010greedy}
P.~Netrapalli, S.~Banerjee, S.~Sanghavi, and S.~Shakkottai.
\newblock Greedy learning of {M}arkov network structure.
\newblock In \emph{48th Allerton Conference}, pages 1295--1302, 2010.

\bibitem[O'Donnell(2014)]{o2014analysis}
R.~O'Donnell.
\newblock \emph{Analysis of boolean functions}.
\newblock Cambridge University Press, 2014.

\bibitem[Ravikumar et~al.(2010)Ravikumar, Wainwright, and
  Lafferty]{ravikumar2010high}
P.~Ravikumar, M.J. Wainwright, and J.D. Lafferty.
\newblock High-dimensional {I}sing model selection using $\ell_1$-regularized
  logistic regression.
\newblock \emph{Annals of Statistics}, 38\penalty0 (3):\penalty0 1287--1319,
  2010.

\bibitem[Ravikumar et~al.(2011)Ravikumar, Wainwright, Raskutti, and
  Yu]{ravikumar2011high}
P.~Ravikumar, M.~Wainwright, G.~Raskutti, and B.~Yu.
\newblock High-dimensional covariance estimation by minimizing
  $\ell_1$-penalized log-determinant divergence.
\newblock \emph{Electronic Journal of Statistics}, 5:\penalty0 935--980, 2011.

\bibitem[Ray et~al.(2012)Ray, Sanghavi, and Shakkottai]{raygreedy}
A.~Ray, S.~Sanghavi, and S.~Shakkottai.
\newblock Greedy learning of graphical models with small girth.
\newblock In \emph{50th Allerton Conference}, 2012.

\bibitem[Ricci-Tersenghi(2012)]{ricci2012bethe}
F.~Ricci-Tersenghi.
\newblock The {B}ethe approximation for solving the inverse {I}sing problem: a
  comparison with other inference methods.
\newblock \emph{Journal of Statistical Mechanics: Theory and Experiment},
  \penalty0 (08), 2012.

\bibitem[Roudi et~al.(2009)Roudi, Aurell, and Hertz]{roudi2009statistical}
Y.~Roudi, E.~Aurell, and J.~Hertz.
\newblock Statistical physics of pairwise probability models.
\newblock \emph{Frontiers in computational neuroscience}, 3, 2009.

\bibitem[Salas and Sokal(1997)]{salas1997absence}
S.~Salas and A.~Sokal.
\newblock Absence of phase transition for antiferromagnetic {P}otts models via
  the {D}obrushin uniqueness theorem.
\newblock \emph{Journal of Statistical Physics}, 86\penalty0 (3-4):\penalty0
  551--579, 1997.

\bibitem[Santhanam and Wainwright(2012)]{santhanam2012information}
N.~P. Santhanam and M.~J. Wainwright.
\newblock Information-theoretic limits of selecting binary graphical models in
  high dimensions.
\newblock \emph{IEEE Trans. on Info. Theory}, 58\penalty0 (7):\penalty0
  4117--4134, 2012.

\bibitem[Schneidman et~al.(2006)Schneidman, Berry, Segev, and
  Bialek]{schneidman2006weak}
E.~Schneidman, M.~Berry, R.~Segev, and W.~Bialek.
\newblock Weak pairwise correlations imply strongly correlated network states
  in a neural population.
\newblock \emph{Nature}, 440\penalty0 (7087):\penalty0 1007--1012, 2006.

\bibitem[Sessak and Monasson(2009)]{sessak2009small}
V.~Sessak and R.~Monasson.
\newblock Small-correlation expansions for the inverse {I}sing problem.
\newblock \emph{Journal of Physics A: Mathematical and Theoretical},
  42\penalty0 (5):\penalty0 055001, 2009.

\bibitem[Sinclair et~al.(2014)Sinclair, Srivastava, and
  Thurley]{sinclair2014approximation}
A.~Sinclair, P.~Srivastava, and M.~Thurley.
\newblock Approximation algorithms for two-state anti-ferromagnetic spin
  systems on bounded degree graphs.
\newblock \emph{Journal of Statistical Physics}, 155\penalty0 (4):\penalty0
  666--686, 2014.

\bibitem[Sly(2010)]{sly2010computational}
A.~Sly.
\newblock Computational transition at the uniqueness threshold.
\newblock In \emph{FOCS}, pages 287--296, 2010.

\bibitem[Sly and Sun(2012)]{sly2012computational}
A.~Sly and N.~Sun.
\newblock The computational hardness of counting in two-spin models on
  d-regular graphs.
\newblock In \emph{FOCS}, pages 361--369. IEEE, 2012.

\bibitem[Srebro(2001)]{srebro2001maximum}
N.~Srebro.
\newblock Maximum likelihood bounded tree-width {M}arkov networks.
\newblock In \emph{UAI}, 2001.

\bibitem[Stroock and Zegarlinski(1992)]{stroock1992logarithmic}
D.~Stroock and B.~Zegarlinski.
\newblock The logarithmic {S}obolev inequality for discrete spin systems on a
  lattice.
\newblock \emph{Comm. in Mathematical Physics}, 149\penalty0 (1):\penalty0
  175--193, 1992.

\bibitem[Tao and Vu(2006)]{tao2006additive}
T.~Tao and V.~Vu.
\newblock \emph{Additive combinatorics}, volume 105.
\newblock Cambridge University Press, 2006.

\bibitem[Toshiyuki(1998)]{tanaka1998mean}
T.~Toshiyuki.
\newblock Mean-field theory of boltzmann machine learning.
\newblock \emph{Physical Review E}, 58\penalty0 (2):\penalty0 2302, 1998.

\bibitem[Valiant(2012)]{valiant2012finding}
G.~Valiant.
\newblock Finding correlations in subquadratic time, with applications to
  learning parities and juntas.
\newblock In \emph{FOCS}, pages 11--20, 2012.

\bibitem[Weigt et~al.(2009)Weigt, White, Szurmant, Hoch, and
  Hwa]{weigt2009identification}
M.~Weigt, R.~White, H.~Szurmant, J.~Hoch, and T.~Hwa.
\newblock Identification of direct residue contacts in protein--protein
  interaction by message passing.
\newblock \emph{Proceedings of the National Academy of Sciences}, 106\penalty0
  (1):\penalty0 67--72, 2009.

\bibitem[Weitz(2006)]{Weitz}
D.~Weitz.
\newblock Counting independent sets up to the tree threshold.
\newblock In \emph{STOC}, 2006.

\bibitem[Wu et~al.(2013)Wu, Srikant, and Ni]{wu2013learning}
R.~Wu, R.~Srikant, and J.~Ni.
\newblock Learning loosely connected markov random fields.
\newblock \emph{Stochastic Systems}, 3\penalty0 (2):\penalty0 362--404, 2013.

\end{thebibliography}
}
\end{document}